\documentclass[11pt,journal,draftcls,a4paper,onecolumn]{IEEEtran}

\newif\ifallannexe

\usepackage{dsfont}
\usepackage{amsmath}
\usepackage{amssymb}
\usepackage{amsfonts}
\usepackage{graphicx}
\usepackage{amsmath}
\usepackage[all,poly]{xy}
\usepackage{multirow}
\usepackage{booktabs}
\usepackage{algorithmic}
\usepackage[algoruled, linesnumbered]{algorithm2e}
\usepackage{color}
\usepackage[noadjust]{cite}
\usepackage{url}
\usepackage{amsthm}

\usepackage[colorinlistoftodos]{todonotes}

\usepackage{bm}
\usepackage{tikz,tkz-tab} 
\usepackage{pgfplots}
\usepackage{color}
\usepackage{chngcntr} 
\usepackage{hhline}

\usepackage[pdftex, linktocpage,bookmarksopen,hidelinks,
    bookmarks           = true,         
    bookmarksopen       = true,
    pdfstartview        = Fit,          
    bookmarksnumbered   = false,         
    pdfpagelayout       = SinglePage,   
    colorlinks          = false,         
    urlcolor            = addresscolor,         
    pdfborder           = {0 0 0},      
    breaklinks          = false,%
    linkcolor           = addresscolor,
    anchorcolor         = addresscolor,%
    citecolor           = addresscolor,
    filecolor           = addresscolor,%
    menucolor           = addresscolor,
    pagecolor           = addresscolor,
    citecolor           = addresscolor
]{hyperref}

\usetikzlibrary{arrows,automata}

\graphicspath{{figures/}}
\newcommand{\figwidth}{0.85\columnwidth}

\newcommand{\headTab}[1]{ \textnormal{\textbf{#1}} }

\newcommand{\refapp}[1]{\cite[App. \ref{#1}]{Elvira2015TR}}

\def\bsI{{\boldsymbol{I}}}

\def\bfe{{\mathbf{e}}}

\def\bfh{{\mathbf{h}}}

\def\bfu{{\mathbf{u}}}

\def\bfx{{\mathbf{x}}}
\def\bfy{{\mathbf{y}}}

\def\bfH{{\mathbf{H}}}

\def\calU{{\mathcal{U}}}

\def\calD{{\mathcal{D}}}

\def\calG{{\mathcal{G}}}

\def\calI{{\mathcal{I}}}

\def\calK{{\mathcal{K}}}

\def\calN{{\mathcal{N}}}

\def\calU{{\mathcal{U}}}

\def\calX{{\mathcal{X}}}

\newcommand{\dimcoef}{N} \newcommand{\dimcoeff}{\dimcoef}
\newcommand{\indcoef}{n} \newcommand{\indcoeff}{\indcoef}
 
\newcommand{\dimobs}{M}
\newcommand{\indobs}{m}

\newcommand{\Vobs}{\bfy} \newcommand{\obs}[1]{y_{#1}}
\newcommand{\Vcoef}{\bfx} \newcommand{\coef}[1]{x_{#1}}

\newcommand{\Vcoeff}{\Vcoef} \newcommand{\coeff}[1]{\coef{#1}}
\newcommand{\MATop}{\bfH}

\newcommand{\Vnoise}{\bfe} \newcommand{\noise}[1]{e_{#1}}

\newcommand{\normq}[1]{\left\|#1\right\|^2_2}
\newcommand{\norminf}[1]{ \left\| #1 \right\|_{\infty} }
\newcommand{\abs}[1]{\left|#1\right|}
\newcommand{\sachant}{|}

\newcommand{\prox}{\textrm{prox}}

\newcommand{\demoparamini}{\lambda}
\newcommand{\demoparam}{\mu}

\newcommand{\indfunc}[2]{\mathbf{1}_{#1}{\left(#2\right)}}

\newcommand{\transp}{^T}

\newcommand{\R}{\mathbb{R}}
\newcommand{\cone}{\mathcal{C}}

\newcommand{\dsgamma}{d\calG}

\newcommand{\noisesymb}{\sigma}
\newcommand{\noisevar}{\noisesymb^2}

\newtheorem{remark}{Remark}
\newtheorem{lemma}{Lemma}
\newtheorem{definition}{Definition}
\newtheorem{property}{Property}

\newcommand{\paramvect}{\boldsymbol{\theta}}

\newcommand{\FITRAa}{{FITRA-mmse}} 
\newcommand{\FITRAc}{{FITRA-snr}} 
\newcommand{\FITRAd}{{FITRA-papr}} 

\newcommand{\mmse}{MMSE}

\newcommand{\mapm}{mMAP}

\newcommand{\MMSE}[1]{\hat{#1}_{\mathrm{MMSE}}}

\newcommand{\MAPm}[1]{\hat{#1}_{\mathrm{mMAP}}}

\newcommand{\Nbi}{T_{\mathrm{bi}}}
\newcommand{\Nmc}{T_{\mathrm{MC}}}
\newcommand{\Nr}{T_{r}}

\newcommand{\snry}{{\text{SNR}}_{\Vobs}}
\newcommand{\snrx}{{\text{SNR}}_{\Vcoef}}
\newcommand{\papr}{{\text{PAPR}}}

\title{Bayesian anti-sparse coding}

\author{Cl\'ement Elvira,~\IEEEmembership{Student Member,~IEEE}, Pierre Chainais,~\IEEEmembership{Member,~IEEE} \\and Nicolas Dobigeon,~\IEEEmembership{Senior Member,~IEEE} \\
\thanks{Cl\'ement Elvira and Pierre Chainais are with Univ. Lille, CNRS, Centrale Lille, UMR 9189 - CRIStAL - Centre de Recherche
en Informatique Signal et Automatique de Lille, F-59000 Lille, France (e-mail:
\{Clement.Elvira, Pierre.Chainais\}@ec-lille.fr ).}
\thanks{Nicolas Dobigeon is with the
University of Toulouse, IRIT/INP-ENSEEIHT, CNRS, 2 rue Charles
Camichel, BP 7122, 31071 Toulouse cedex 7, France (e-mail: Nicolas.Dobigeon@enseeiht.fr).}
\thanks{Part of this work has been funded thanks to the BNPSI ANR project no. ANR-13-BS-03-0006-01.}
}

\begin{document}

\maketitle

\begin{abstract}

Sparse representations have proven their efficiency in solving a wide class of inverse problems encountered in signal and image processing. Conversely, enforcing the information to be spread uniformly over representation coefficients exhibits relevant properties in various applications such as digital communications. Anti-sparse regularization can be naturally expressed through an $\ell_{\infty}$-norm penalty. This paper derives a probabilistic formulation of such representations. A new probability distribution, referred to as the democratic prior, is first introduced. Its main properties as well as three random variate generators for this distribution are derived. Then this probability distribution is used as a prior to promote anti-sparsity in a Gaussian linear inverse problem, yielding a fully Bayesian formulation of anti-sparse coding. Two Markov chain Monte Carlo (MCMC) algorithms are proposed to generate samples according to the posterior distribution. The first one is a standard Gibbs sampler. The second one uses Metropolis-Hastings moves that exploit the proximity mapping of the log-posterior distribution. These samples are used to approximate maximum a posteriori and minimum mean square error estimators of both parameters and hyperparameters. Simulations on synthetic data illustrate the performances of the two proposed samplers, for both complete and over-complete dictionaries. All results are compared to the recent deterministic variational FITRA algorithm.

\end{abstract}


\newpage
\section{Introduction}\label{sec:introduction}

\IEEEPARstart{S}{parse} representations have been widely advocated for as an efficient tool to address various problems encountered in signal and image processing. As an archetypal example, they were the core concept underlying most of the lossy data compression schemes, exploiting compressibility properties of natural signals and images over appropriate bases. Sparse approximations, generally resulting from a \emph{transform coding} process, lead for instance to the famous image, audio and video compression standards JPEG, MP3 and MPEG \cite{Candes2008,Davenport2012}. More recently and partly motivated by the advent of both the compressive sensing and dictionary learning paradigms, sparsity has been intensively exploited to regularize (e.g., linear) ill-posed inverse problems. The $\ell_0$-norm and the $\ell_1$-norm as its convex relaxation are among the most popular sparsity promoting penalties. Following the ambivalent interpretation of penalized regression optimization \cite{Gribonval2011}, Bayesian inference naturally offers an alternative and flexible framework to derive estimators associated with sparse coding problems. For instance, it is well known that a straightforward Bayesian counterpart of the LASSO shrinkage operator \cite{Tibshirani1996} can be obtained by adopting a Laplace prior \cite{Babacan2010}. Designing other sparsity inducing priors has motivated numerous research works. They generally rely on hierarchical mixture models \cite{Caron2008nips,Tipping2001,Lee2010,Themelis2012}, heavy tail distributions \cite{Ji2008,Tzikas2009,Fevotte2006} or Bernoulli-compound processes \cite{Dobigeon2009ip,Soussen2011,Chaari2014icassp}.

In contrast, the use of the $\ell_\infty$-norm within an objective criterion has remained somehow confidential in the signal processing literature. One may cite the minimax or Chebyshev approximation principle, whose practical implementation has been made possible thanks to the Remez exchange algorithm \cite{Remes1936} and leads to a popular design method of finite impulse response digital filters \cite{Parks1972,McClellan2005}. Besides, when combined with a set of linear equality constraints, minimizing a $\ell_\infty$-norm is referred to as the minimum-effort control problem in the optimal control framework \cite{Neustadt1962,Cadzow1971}. Much more recently, a similar problem has been addressed by Lyubarskii \emph{et al.} in \cite{Lyubarskii2010} where the \emph{Kashin's representation} of a given vector over a tight frame is introduced as the expansion coefficients with the smallest possible dynamic range. Spreading the information over representation coefficients in the most uniform way is a desirable feature in various applicative contexts, e.g., to design robust analog-to-digital conversion schemes \cite{Cvetkovic2003,Calderbank2002} or to reduce the peak-to-average power ratio (PAPR) in multi-carrier transmissions \cite{Farrell2009,Ilic2009}. Resorting to an uncertainty principle (UP), Lyubarskii \emph{et al.} have also introduced several examples of frames yielding computable Kashin's representations, such as random orthogonal matrices, random subsampled discrete Fourier transform (DFT) matrices, and random sub-Gaussian matrices \cite{Lyubarskii2010}. The properties of the alternate optimization problem, which consists of minimizing the maximum magnitude of the representation coefficients for an upper-bounded $\ell_2$-reconstruction error, have been deeply investigated in \cite{Studer2012Allerton,Studer2015sub}. In these latest contributions, the optimal expansion is called the \emph{democratic representation} and some bounds associated with archetypal matrices ensuring the UP are derived. In \cite{Fuchs2011asilomar}, the constrained signal representation problems considered in \cite{Lyubarskii2010} and \cite{Studer2015sub} is converted into their penalized counterpart. More precisely, the so-called \emph{spread} or \emph{anti-sparse representations} result from a variational optimization problem where the admissible range of the coefficients has been penalized through a $\ell_\infty$-norm
\begin{equation}
\label{eq:problem}
  \min_{\Vcoef \in \mathbb{R}^\dimcoef} \frac{1}{2\noisevar}\normq{\Vobs-\MATop\Vcoef} + \lambda \norminf{\Vcoef}.
\end{equation}
In \eqref{eq:problem}, $\MATop$ defines the $\dimobs\times\dimcoef$ representation matrix and $\noisevar$ stands for the variance of the residual resulting from the approximation. Again, the anti-sparse property brought by the $\ell_\infty$-norm penalization enforces the information brought by the measurement vector $\Vobs$ to be evenly spread over the representation coefficients in $\Vcoeff$ with respect to the dictionary $\MATop$. It is worth noting that recent applications have capitalized on these latest theoretical and algorithmic advances, including approximate nearest neighbor search \cite{Jegou2012icassp} and PAPR reduction \cite{Studer2013jsac}.

Surprisingly, up to our knowledge, no probabilistic formulation of these democratic representations has been proposed in the literature. The present paper precisely attempts to fill this gap by deriving a Bayesian formulation of the anti-sparse coding problem \eqref{eq:problem} considered in \cite{Fuchs2011asilomar}. Note that this objective differs from the contribution in \cite{Tan2014} where a Bayesian estimator associated with an $\ell_\infty$-norm loss function has been introduced. Instead, we merely introduce a Bayesian counterpart of the variational problem \eqref{eq:problem}. The main motivations for deriving the proposed Bayesian strategy for anti-sparse coding are threefold. Firstly, Bayesian inference is a flexible methodology that may allow other parameters and hyperparameters (e.g., residual variance $\noisevar$, regularization parameters $\lambda$) to be jointly estimated with the parameter of interest $\Vcoeff$. Secondly, through the choice of the considered Bayes risk, it permits to define a wide class of estimators, beyond the traditional penalized maximum likelihood estimator resulting from the solution of \eqref{eq:problem}. Finally, within this framework, Markov chain Monte Carlo algorithms can be conveniently designed to generate samples according to the posterior distribution. Contrary to deterministic optimization algorithms which provide only one point estimate, these samples can be subsequently used to build a comprehensive statistical description of the solution.

To this purpose, a new probability distribution as well as its main properties are introduced in Section~\ref{sec:democratic}. In particular, we show that $p\left(\Vcoef\right) \propto \exp\left(-\lambda \norminf{\Vcoef}\right)$ properly defines a probability density function (pdf). In Section \ref{sec:inverse_pb}, this so-called \emph{democratic distribution} is used as a prior distribution in a linear Gaussian inverse problem, which provides a straightforward equivalent of the problem \eqref{eq:problem} under the maximum \emph{a posteriori} paradigm. Moreover, exploiting  relevant properties of the democratic distribution, this section describes two Markov chain Monte Carlo (MCMC) algorithms as alternatives to the deterministic solvers proposed in \cite{Fuchs2011asilomar,Studer2015sub}. The first one is a standard Gibbs sampler which sequentially generates samples according to the conditional distributions associated with the joint posterior distribution. The second MCMC algorithm relies on a proximal Monte Carlo step recently introduced in \cite{Pereyra2015prox}. This step exploits the proximal operator associated to the logarithm of the target distribution to sample random vectors asymptotically distributed according to this non-smooth density. Section \ref{sec:experiments} illustrates the performances of the proposed algorithms on numerical experiments. Concluding remarks are reported in Section \ref{sec:conclusion}.

\begin{table}[h!]
\renewcommand{\arraystretch}{1.1}
\renewcommand\tabcolsep{1pt}
	\caption{List of symbols.}
	\label{tab:list_Symbols}
	\centering
	\begin{tabular}{c l}
		\toprule[1.5pt]
		\textbf{Symbol} & \textbf{Description}   \\
		\midrule
		$\dimcoef$, $\indcoef$ 		& Dimension, index of representation vector \\
		$\dimobs$, $\indobs$  		& Dimension, index of observed vector \\
		$\Vcoef$, $\coef{\indcoef}$ & Representation vector, its $\indcoef^{\text{th}}$ component  \\
		$\Vobs$, $\obs{\indobs}$ 	& Observation vector, its $\indobs^{\text{th}}$ component \\
		$\MATop$ 					& Coding matrix \\
		$\Vnoise$   				& Additive noise vector\\
		$\lambda$ 					& Parameter of the democratic distribution \\
		$\mu$ 						& Re-parametrization of $\lambda$ such that $\lambda = \dimcoef \mu$ \\
		$\calD_\dimcoef(\lambda)$ 	& Democratic distribution of parameter $\lambda$ over $\R^{\dimcoef}$ \\
		$C_\dimcoef(\lambda)$ 		& Normalizing constant of the distribution $\calD_\dimcoef(\lambda)$\\
		$\calK_J$ 					& A $J$-element subset $\left\{ i_1 \dots i_J \right\}$ of $\left\{1,\ldots, \dimcoef\right\}$ \\
		$\calU$, $\calG$, $\calI\calG$	& Uniform, gamma and inverse gamma distributions \\
		$\dsgamma$ 					& Double-sided gamma distribution \\
		$ \calN_{\calI}$		    & Truncated Gaussian distribution over $\calI$ \\
		$\cone_{\indcoef}$			& Double convex cones partitioning $\R^{\dimcoef}$\\
		\multirow{2}{*}{$c_\indcoeff$, $\calI_{\indcoeff}$}  & Weights and intervals defining the \\
                                    & \qquad conditional distribution $p\left(\coef{\indcoef} | \Vcoeff_{\backslash\indcoef}\right)$\\
		$g$, $g_1$, $g_2$			& Negative log-distribution ($g = g_1+g_2$) \\
		$\delta$					& Parameter of the proximity operator \\
		\multirow{2}{*}{$\varepsilon_j$, $d_j$,   $\phi_{\delta}(\Vcoeff)$} & Family of distinct values of $\vert \Vcoef \vert$, their respective multiplicity  \\
									& \qquad and family of local maxima of $\prox_{\lambda \norminf{\cdot}}^{\delta}$\\
		$q(\cdot|\cdot)$			& Proposal distribution	\\
		\multirow{2}{*}{$\omega_{i\indcoef}, \mu_{i\indcoef}, s_{\indcoef}^2$, $\calI_{i\indcoef}$} & Weights, parameters and intervals defining the \\
                                    & \qquad conditional  distribution $p\left(\coef{\indcoef} | \Vcoeff_{\backslash\indcoef},\demoparam,\noisevar,\Vobs\right)$ \\
		\bottomrule[1.5pt]
	\end{tabular}
\end{table}

\section{Democratic distribution}
\label{sec:democratic}

This section introduces the democratic distribution and the main properties related to its marginal and conditional distributions. Finally, two random variate generators are proposed. Note that, for sake of conciseness, the proofs associated with the following results are reported in the companion report \cite{Elvira2015TR}.

\subsection{Probability density function}

\begin{lemma}
\label{lem:integrability}
	Let $\Vcoef\in\mathbb{R}^{\dimcoef}$ and $\lambda\in\mathbb{R}_+$. The integral of the function $\exp\left(-\lambda\norminf{\Vcoef}\right)$ over $\mathbb{R}^{\dimcoef}$ is properly defined and the following equality holds
	\begin{equation*}
	 \int_{\mathbb{R}^\dimcoef}\exp\left(-\lambda\norminf{\Vcoef}\right)d\Vcoef = \dimcoef!\left(\frac{2}{\lambda}\right)^\dimcoef.
	\end{equation*}
\end{lemma}

\begin{proof}
\ifallannexe
	See Appendix~\refapp{app:integrability}. 
\else
	See Appendix~\cite[App. A]{Elvira2015TR}.
\fi

\end{proof}

As a corollary of  Lemma \ref{lem:integrability}, the democratic distribution can be defined as follows.
\begin{definition}
	A $\dimcoef$-real-valued random vector $\Vcoef\in\mathbb{R}^{\dimcoef}$ is said to be distributed according to the democratic distribution $\calD_\dimcoef(\lambda)$, namely $\Vcoef\sim\calD_\dimcoef(\lambda)$, when the corresponding pdf is
	\begin{equation}
		\label{eq:democratic}
		p\left(\Vcoef\right) = \frac{1}{ C_\dimcoef(\lambda)} \exp\left(-\lambda\norminf{\Vcoef}\right)
	\end{equation}
	with $C_\dimcoef(\lambda)\triangleq \dimcoef! \left(\frac{2}{\lambda}\right)^\dimcoef $.
\end{definition}

\begin{figure}[h!]
\centering
 \includegraphics[width=0.45\textwidth, height=0.22\textwidth]{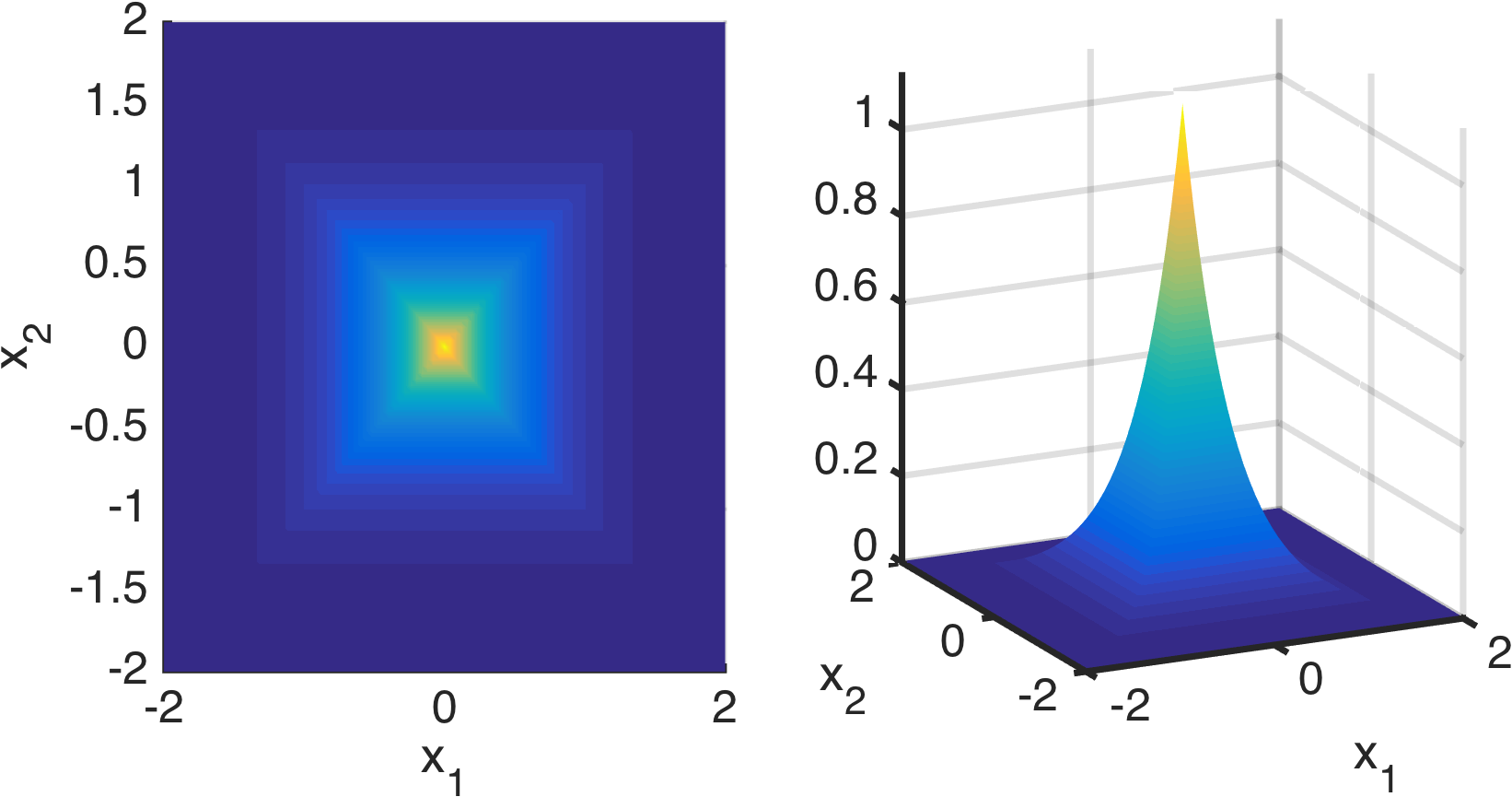}
\caption{The democratic pdf $\calD_\dimcoeff\left(\lambda\right)$ for $\dimcoeff=2$ and $\lambda=3$.} \label{fig:democratic_pdf}
\end{figure}

As an illustration, the pdf of the bidimensional democratic pdf for $\lambda=3$ is depicted in Fig. \ref{fig:democratic_pdf}.

\begin{remark}
It is interesting to note that the democratic distribution belongs to the exponential family. Indeed, its pdf can be factorized as
\begin{equation}
	p\left( \Vcoef \right) = a(\Vcoef) \; b(\lambda) \; \exp \left( \eta(\lambda) T(\Vcoef) \right)
\end{equation}
where $a(\Vcoef) = 1$, $b(\lambda) = 1\slash C_\dimcoef(\lambda)$, $\eta(\lambda) = - \lambda$ and $T(\Vcoef) = \norminf{\Vcoef}$ defines sufficient statistics.
\end{remark}

\subsection{Moments}

The two first moments of the democratic distribution are available through the following property.

\begin{property}
Let $\Vcoef = \left[\coef{1},\ldots,\coef{\dimcoef}\right]\transp$ be a random vector obeying the democratic distribution  $\calD_\dimcoef(\lambda)$. The mean and the covariance matrix are given by:
\begin{align}
\mathrm{E} \left[ \coef{\indcoef} \right] & = 0& \quad \forall \indcoef \in \{1, \ldots,\dimcoef\}\label{eq:mean} \\
 \mathrm{var}\left[ \coef{\indcoef} \right] & = \frac{(\dimcoef+1)(\dimcoef+2)}{3 \lambda^2} &\quad \forall \indcoef \in \{1, \ldots,\dimcoef\} \label{eq:var} \\
	\mathrm{cov}\left[ \coef{i}, \coef{j} \right] &= 0 \quad &\forall i\neq j. \label{eq:cov}
\end{align}

\end{property}

\begin{proof}

\ifallannexe
	See Appendix~\ref{app:moments}.
\else
	See Appendix~\cite[App. B]{Elvira2015TR}.
\fi
\end{proof}

\subsection{Marginal distributions}

The marginal distributions of any democratically distributed vector $\Vcoeff$ are given by the following Lemma.

\begin{lemma} \label{lem:marginal_subset}
	Let $\Vcoef = \left[\coef{1},\ldots,\coef{\dimcoef}\right]\transp$ be a random vector obeying the democratic distribution  $\calD_\dimcoef(\lambda)$. For any positive integer $J<\dimcoeff$, let $\calK_J$  
		denote a $J$-element subset of $\left\{1,\ldots,\dimcoeff\right\}$  and $\Vcoef_{\backslash \calK_J}$ the sub-vector of $\Vcoeff$ whose $J$ elements indexed by $\calK_J$ have been removed. Then the marginal pdf of the sub-vector $\Vcoef_{\backslash\calK_J} \in \R^{\dimcoeff-J}$ is given by
	\begin{multline}
		p\left(\Vcoef_{\backslash\calK_J}\right) = \frac{2^J}{C_{\dimcoeff}(\lambda)} \sum_{j=0}^{J} \binom{J}{j} \frac{(J-j)!}{\lambda^{J-j}} \norminf{\Vcoef_{\backslash\calK_J}}^j \\ \times \exp\left(-\lambda\norminf{\Vcoef_{\backslash\calK_J}}\right).
	\end{multline}
\end{lemma}

\begin{proof}
\ifallannexe
	See Appendix \ref{app:marginal}.
\else
	\cite[App. C]{Elvira2015TR}
\fi
\end{proof}

In particular, as a straightforward corollary of this lemma, two specific marginal distributions of $\calD_\dimcoeff(\lambda)$ are given by the following property.

\begin{property}
Let $\Vcoef = \left[\coef{1},\ldots,\coef{\dimcoef}\right]\transp$ be a random vector obeying the democratic distribution  $\calD_\dimcoef(\lambda)$. The components $\coeff{\indcoef}$ ($\indcoeff=1,\ldots,\dimcoeff$) of $\Vcoef$ are identically and marginally distributed according to the following $\dimcoeff$-component mixture of double-sided Gamma distributions\footnote{The double-sided Gamma distribution $\dsgamma\left(a,b\right)$ is defined as a generalization over $\mathbb{R}$ of the standard Gamma distribution $\calG\left(a,b\right)$ with the  pdf $p(x) = \frac{b^{a}}{2\Gamma(b)}\abs{x}^{a-1} \exp\left(-b \abs{x}\right)$.}
			\begin{equation}
			\label{eq:marginal}
				\coef{\indcoef}  \sim \frac{1}{\dimcoef}\sum_{j=1}^{\dimcoef} \dsgamma\left(j,\lambda\right).
			\end{equation}
Moreover, the pdf of the sub-vector  $\Vcoef_{\backslash\indcoef}$ of $\Vcoeff$ whose $\indcoeff$th element has been removed is :
\begin{equation} \label{eq:marginalJustOneElement}
	p\left(\Vcoef_{\backslash\indcoef}\right) =   \frac{1 + \lambda \norminf{\Vcoef_{\backslash\indcoef}} }{\dimcoeff \;C_{\dimcoeff-1}(\lambda)}
		\exp\left(-\lambda\norminf{\Vcoef_{\backslash\indcoef}}\right).
\end{equation}
\end{property}

\begin{proof}
\ifallannexe
	See Appendix \ref{app:marginal}.
\else
	\cite[App. C]{Elvira2015TR}
\fi
\end{proof}

These two specific marginal distributions $p\left(\Vcoef_{\backslash\indcoef}\right)$ and $p\left(\coeff{\indcoeff}\right)$ are depicted in Fig.~\ref{fig:fig_democratic_pdf_marginal_subvector} (top and bottom, right).

\begin{figure}[h!]
\centering
 \includegraphics[width=0.45\textwidth, height=0.22\textwidth]{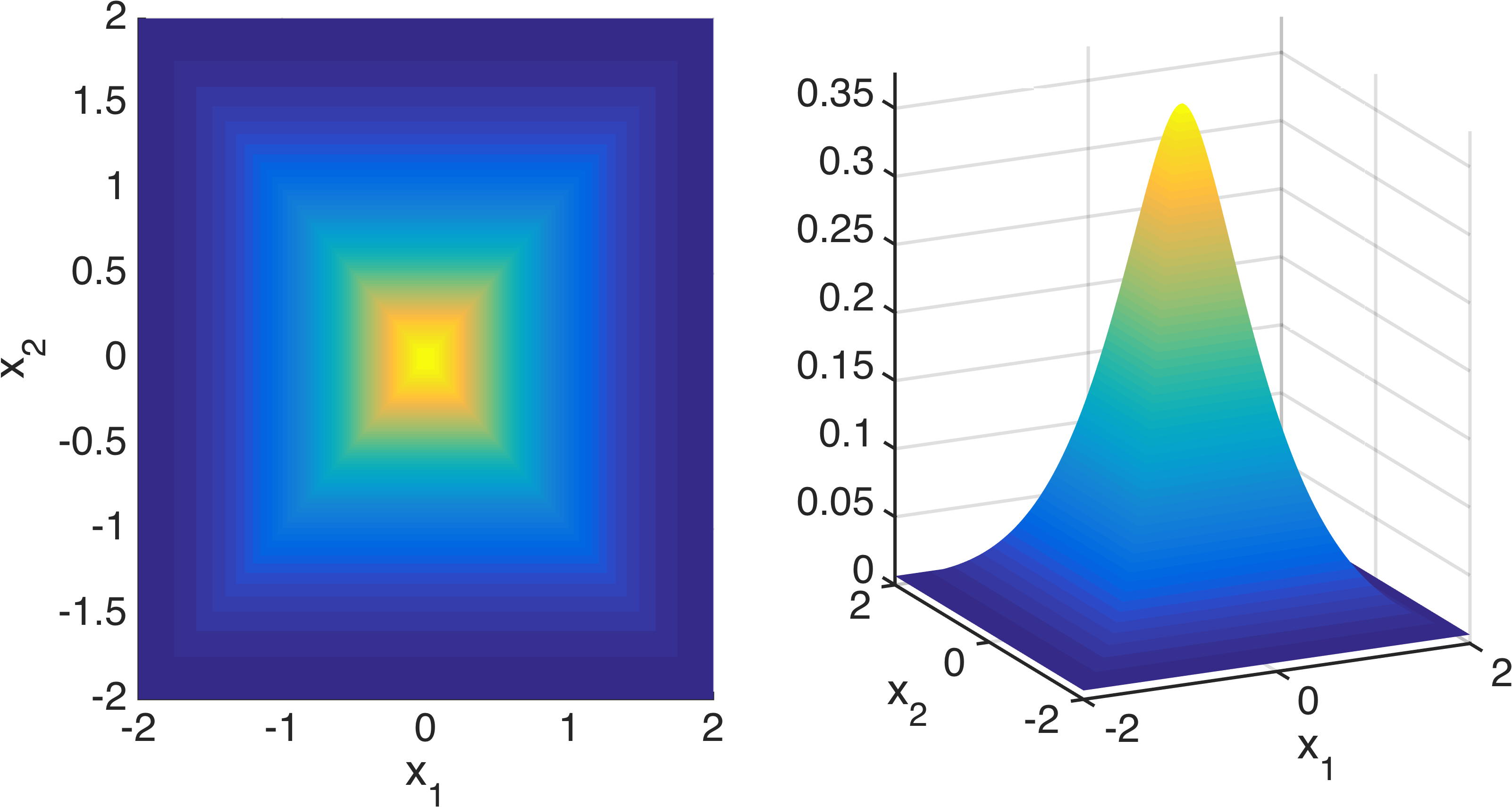}\\
 \includegraphics[width=0.45\textwidth, height=0.22\textwidth]{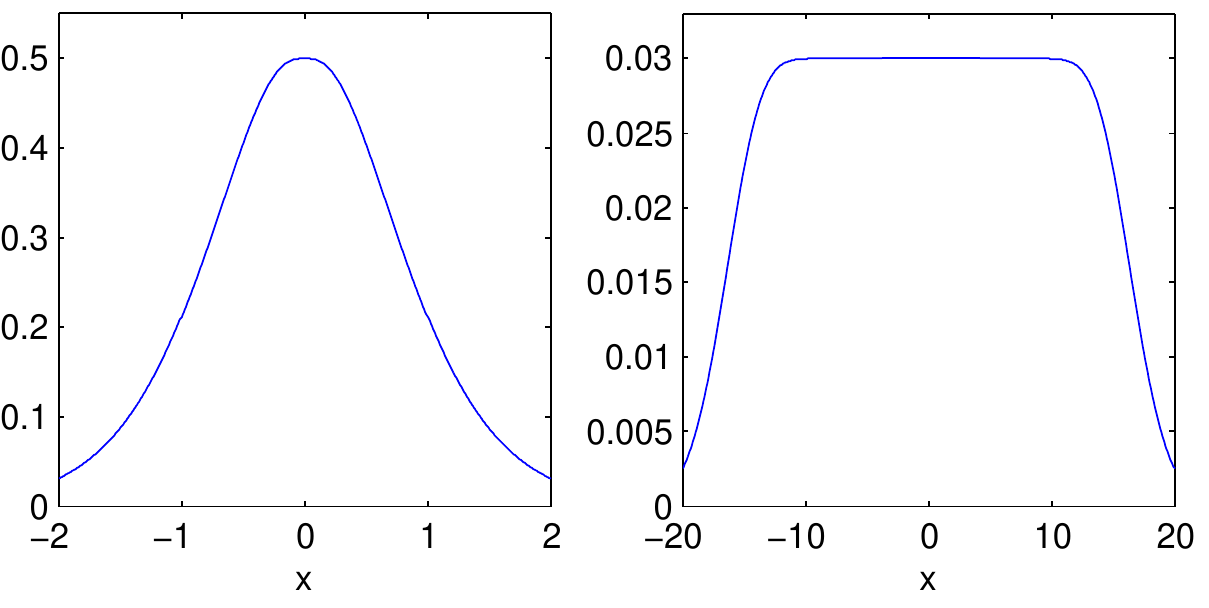}
\caption{Top: marginal distribution of $\Vcoef_{\backslash\indcoef}$ when $\Vcoeff\sim\calD_\dimcoeff\left(\lambda\right)$ for $\dimcoeff=3$ and $\lambda=3$.
Bottom: marginal distribution of $\coef{\indcoef}$ when $\Vcoeff\sim\calD_\dimcoeff\left(\lambda\right)$ for $\dimcoeff=3$ (left) or $\dimcoeff=50$ (right) and  $\lambda=3$.} \label{fig:fig_democratic_pdf_marginal_subvector}
\end{figure}

\begin{remark}
	It is worth noting that the distribution in \eqref{eq:marginal} can be rewritten as
	\begin{equation*}
		p\left( \coef{\indcoef}\right) = \frac{\lambda}{2\dimcoeff} \left(\sum_{j=0}^{\dimcoeff-1}\frac{\lambda^j}{j!}\abs{\coeff{\indcoeff}}^j\right)\exp\left(-\lambda\abs{\coeff{\indcoeff}}\right)
	\end{equation*}
	which behaves as $p\left( \coef{\indcoef}\right)\approx \frac{\lambda}{2\dimcoef}$ when $\dimcoef \rightarrow +\infty$. This means that the components of $\Vcoef$ tend to be marginally distributed according to uniform distributions over $\mathbb{R}$ in high dimension. This behavior is depicted in Fig.~\ref{fig:fig_democratic_pdf_marginal_subvector} bottom-right.
\end{remark}

\subsection{Conditional distributions}

Before introducing conditional distributions associated with any democratically distributed random vector, let partition $\R^{\dimcoeff}$ into a set of $\dimcoeff$ non-overlapping double-convex cones $\cone_{\indcoeff} \subset \R^{\dimcoeff}$ ($\indcoeff=1,\ldots,\dimcoeff$) defined by
\begin{equation}
\label{eq:cones}
	\cone_{\indcoeff} \triangleq \left\{\bfx=[x_1,\ldots,x_{\dimcoeff}]^T \in  \R^{\dimcoeff}\  : \
	 \forall j\neq \indcoeff,\  \abs{x_\indcoeff} > \abs{x_j}\right\}.
\end{equation}
These sets are directly related to the index of the so-called \emph{dominant component} of a given democratically distributed vector $\Vcoeff$. More precisely, if $\norminf{\bfx}=\abs{x_{\indcoeff}}$, then $\bfx \in \cone_{\indcoeff}$ and the $\indcoeff^{\text{th}}$ component $x_{\indcoeff}$ of $\bfx$ is said to be the dominant component.

An example is given in Fig. \ref{fig:cone} where $\cone_{1} \subset \R^{2}$ is depicted. These double-cones partition $\mathbb{R}^\dimcoeff$ into $\dimcoeff$ equiprobable sets with respect to (w.r.t.) the democratic distribution, as stated in the following property.

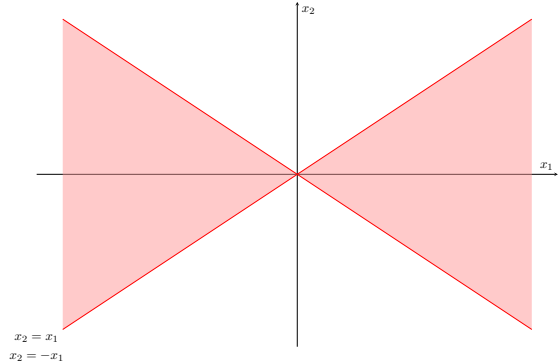
\begin{figure}[h!]
\centering
 \scalebox{.45}{\definecolor{mycolor1}{rgb}{0.00000,0.44700,0.74100}%
\definecolor{mycolor2}{rgb}{0.00000,0.75000,1.00000}%
\definecolor{mycolor3}{rgb}{0.89000,0.44000,0.48000}%

\definecolor{lightred}{rgb}{1.00,0.59,0.59}%

\definecolor{alizarin}{rgb}{0.82, 0.1, 0.26}
\definecolor{denim}{rgb}{0.08, 0.38, 0.74}

\begin{tikzpicture}[
        hatch distance/.store in=\hatchdistance,
        hatch distance=10pt,
        hatch thickness/.store in=\hatchthickness,
        hatch thickness=2pt
    ]

    \makeatletter
    \pgfdeclarepatternformonly[\hatchdistance,\hatchthickness]{flexible hatch}
    {\pgfqpoint{0pt}{0pt}}
    {\pgfqpoint{\hatchdistance}{\hatchdistance}}
    {\pgfpoint{\hatchdistance-1pt}{\hatchdistance-1pt}}%
    {
        \pgfsetcolor{\tikz@pattern@color}
        \pgfsetlinewidth{\hatchthickness}
        \pgfpathmoveto{\pgfqpoint{0pt}{0pt}}
        \pgfpathlineto{\pgfqpoint{\hatchdistance}{\hatchdistance}}
        \pgfusepath{stroke}
    }

\begin{axis}[%
width=6in,
height=4in,
at={(0, 0)},
scale only axis,
unbounded coords=jump,
clip=false,
xmin=-1,
xmax=1,
ymin=-1,
ymax=1,
xtick={\empty},
ytick={\empty},
xlabel={$x_1$},
ylabel={$x_2$},
axis lines=middle
]

\addplot+[mark=none,
	domain=-0.9:0.9,
	samples=100,
    fill=lightred,
    draw=none,
    opacity=0.5,
	]{x}\closedcycle;
\addplot+[mark=none,
	domain=-0.9:0.9,
	samples=100,
    fill=lightred,
    draw=none,
    opacity=0.5,
	]{-x}\closedcycle;

\addplot [color=red,solid,forget plot] table[row sep=crcr]{%
-0.9	-0.9 \\
0.9	0.9 \\
} node[at end, above, black] {$\coeff{2}=\coeff{1}$};

\addplot [color=red,solid,forget plot] table[row sep=crcr]{%
-0.9	0.9 \\
0.9	-0.9 \\
} node[at end, below, black] {$\coeff{2}=-\coeff{1}$};





\end{axis}
\end{tikzpicture}%
}
\caption{The double-cone $\cone_{1}$ of $\R^2$ appears as the light red area while the complementary double-cone $\cone_{2}$ is the uncolored area.} \label{fig:cone}
\end{figure}

\begin{property}
\label{prop:dominant_proba}
	Let $\Vcoef = \left[\coef{1},\ldots,\coef{\dimcoef}\right]\transp$ be a random vector obeying the democratic distribution  $\calD_\dimcoef(\lambda)$. Then the probability that this vector belongs to a given double-cone is
	\begin{equation}
	\label{eq:dominant_proba}
		\mathrm{P}\left[\Vcoeff \in \cone_{\indcoeff}\right] = \frac{1}{\dimcoeff}.
	\end{equation}
\end{property}

\begin{proof}
\ifallannexe
	See Appendix \ref{app:conditional}, paragraph \ref{app:conditional:sec:dominant_proba}.
\else
	See Appendix~\cite[App. D]{Elvira2015TR} paragraph D-A.
\fi
\end{proof}

\begin{remark}
	This property simply exhibits intuitive intrinsic symmetries of the democratic distribution: the dominant component of a democratically distributed vector is located with equal probabilities in any of the cones $\cone_{\indcoeff}$.
\end{remark}

Moreover, the following lemma yields some results on conditional distributions related to these sets.

\begin{lemma}
\label{lem:dominant_pdf}
	Let $\Vcoef = \left[\coef{1},\ldots,\coef{\dimcoef}\right]\transp$ be a random vector obeying the democratic distribution  $\calD_\dimcoef(\lambda)$. Then the following results hold
	\begin{eqnarray}
		\coeff{\indcoeff}|\Vcoeff \in \cone_{\indcoeff} &\sim& \dsgamma\left(\dimcoeff,\lambda\right) \label{eq:dominant_pdf}\\
		\Vcoef_{\backslash\indcoef}| \Vcoeff \in \cone_{\indcoeff} &\sim& \calD_{\dimcoeff-1}(\lambda) \label{eq:nondominant_pdf}\\
		\coeff{j}|\coeff{\indcoeff}, \Vcoeff \in \cone_{\indcoeff} &\sim& \calU\left(-\abs{\coeff{\indcoeff}},\abs{\coeff{\indcoeff}}\right)\quad (j\neq \indcoeff) \label{eq:dominant_proba_cond}
	\end{eqnarray}
	and
	\begin{eqnarray}
		\mathrm{P}\left[\Vcoeff \in \cone_{\indcoeff} | \Vcoef_{\backslash\indcoef}\right] &=& \frac{1}{1+\lambda \norminf{\Vcoef_{\backslash\indcoef}}} \label{eq:dominant_proba_cond_nondominant}\\
		p\left(\Vcoef_{\backslash\indcoef}|\Vcoeff \not\in \cone_{\indcoeff} \right) &=& \frac{\lambda }{\dimcoef-1}  \frac{\norminf{\Vcoef_{\backslash\indcoef}}}{C_{\dimcoeff-1}(\lambda)} 
		e^{-\lambda \norminf{\Vcoef_{\backslash\indcoef}}}. \label{eq:proba_marg_noncone}
	\end{eqnarray}
\end{lemma}

\begin{proof}
\ifallannexe
	See Appendix \ref{app:conditional}, paragraph \ref{app:conditional:sec:dominant_pdf}.
\else
	See Appendix~\cite[App. D]{Elvira2015TR} paragraph D-B.
\fi
\end{proof}

\begin{remark}
	 According to \eqref{eq:dominant_pdf}, the marginal distribution of the dominant component is a double-sided Gamma distribution. Conversely, according to \eqref{eq:nondominant_pdf}, the vector of the non-dominant components is marginally distributed according to a democratic distribution. Conditionally upon the dominant component, these non-dominant components are independently and uniformly distributed on the admissible set, as shown in \eqref{eq:dominant_proba_cond}. The probability in \eqref{eq:dominant_proba_cond_nondominant} shows that the probability that the $\indcoeff$th component dominates increases when the other components are of low amplitude.
\end{remark}

Finally, based on Lemma~\ref{lem:dominant_pdf}, the following property related to the conditional distributions of $\calD_{\dimcoef}(\lambda)$ can be stated.

\begin{property}
\label{prop:conditional}
 Let $\Vcoef = \left[\coef{1},\ldots,\coef{\dimcoef}\right]\transp$ be a random vector obeying the democratic distribution  $\calD_\dimcoef(\lambda)$. The pdf of the conditional distribution of a given component $\coef{\indcoef}$ given $\Vcoeff_{\backslash\indcoef}$ is
			\begin{multline}
			\label{eq:conditional}
				p\left(\coef{\indcoef} | \Vcoeff_{\backslash\indcoef}\right) =
				\left(1-c_\indcoeff\right)\frac{1}{2\norminf{\Vcoeff_{\backslash \indcoeff}}}\indfunc{\calI_{\indcoeff}}{\coef{\indcoef}}\\
				 + c_\indcoeff
				 \frac{\lambda}{2}e^{-\lambda\left(\abs{\coef{\indcoef}}+\norminf{\Vcoeff_{\backslash \indcoeff}}\right)}\indfunc{\mathbb{R}\backslash {\calI_{\indcoeff}}}{\coef{\indcoef}}
			\end{multline}
			with $\calI_{\indcoeff} \triangleq \left(-\norminf{\Vcoeff_{\backslash \indcoeff}},\norminf{\Vcoeff_{\backslash \indcoeff}}\right)$ and where $c_\indcoeff= \mathrm{P}\left[\Vcoeff \in \cone_{\indcoeff} | \Vcoef_{\backslash\indcoef}\right]$ is defined by \eqref{eq:dominant_proba_cond_nondominant}.
\end{property}

\begin{proof}
\ifallannexe
	See Appendix \ref{app:conditional}, paragraph \ref{app:conditional:sec:conditional}.
\else
	See Appendix \cite[App. D]{Elvira2015TR} paragraph D-C.
\fi
\end{proof}

\begin{remark}
\label{rem:conditional}
	The pdf in \eqref{eq:conditional} defines a mixture of one uniform distribution and two shifted exponential distributions with probabilities $1-c_\indcoeff$ and $c_\indcoeff\slash 2$, respectively. An example of this pdf is depicted in Fig \ref{fig:democratic_pdf_conditional_single}. This property opens the door to a natural random variate generator according to the democratic distribution through the use of a standard Gibbs sampler. This random generation strategy is detailed in paragraph \ref{subsubsec:generator_gibbs}.
\end{remark}

\begin{figure}[h!]
\centering
 \includegraphics[width=\figwidth]{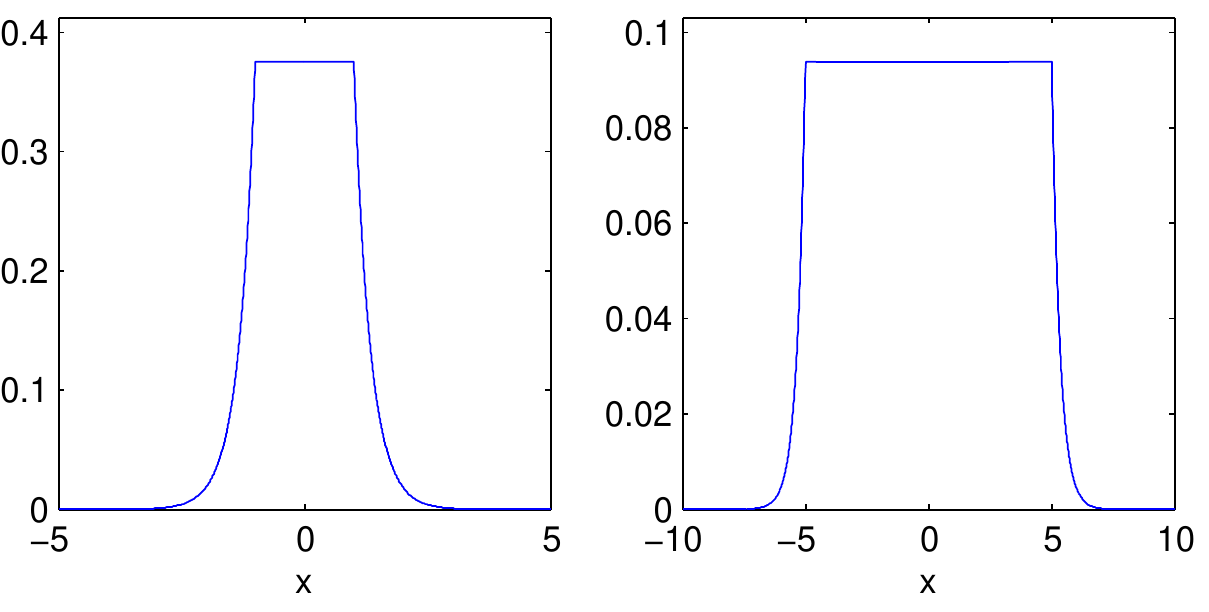}
\caption{Conditional distribution of $\coeff{\indcoeff}|\Vcoef_{\backslash\indcoef}$ when $\Vcoeff\sim\calD_\dimcoeff\left(\lambda\right)$ for $\dimcoeff=2$, $\lambda=3$ and $\norminf{\Vcoef_{\backslash\indcoef}}=1$ (left) or $\norminf{\Vcoef_{\backslash\indcoef}}=10$ (right).} \label{fig:democratic_pdf_conditional_single}
\end{figure}

\subsection{Proximity operator of the negative log-pdf}
\label{subsec:proximity_operator}

The pdf of the democratic distribution $\calD_{\dimcoeff}(\lambda)$ can be written as $p(\Vcoeff) \propto \exp\left(-g_1\left(\Vcoeff\right)\right)$ with
\begin{equation}
\label{eq:definition_g1}
   \quad g_1\left(\Vcoeff\right) = \lambda \norminf{\Vcoeff}.
\end{equation}
This paragraph introduces the proximity mapping operator associated with the negative log-distribution $g_1(\Vcoeff)$ (defined up to a multiplicative constant). This proximal operator will be subsequently resorted to implement Monte Carlo algorithms to draw samples from the democratic distribution $\calD_{\dimcoeff}(\lambda)$ (see paragraph \ref{subsec:random_variate_generation}) as well as posterior distributions derived from a democratic prior (see paragraph \ref{subsubsec:gibbs_description_vector_PMALA}). In this context, it is convenient to define the proximity operator of $g_1\left(\cdot\right)$ as \cite{moreau1962fonctions}
\begin{equation} \label{eq:def_proximal_operator}
\prox_{ g_1 }^{\delta} (\Vcoef ) = \underset{\bfu \in \R^{\dimcoef}}{\text{argmin}} \; \lambda\norminf{\bfu} + \frac{1}{2 \delta} \Vert \Vcoeff-\bfu \Vert_2^2
\end{equation}
Up to the authors' knowledge, no closed-form expression is available for \eqref{eq:def_proximal_operator}. However, this operator can be explicitly computed following the algorithmic scheme detailed in Algo.~\ref{alg:prox}, based on the following property.

\begin{property}
\label{prop:proximal}
Let $\Vcoef\in \R^{\dimcoef}$ and $\delta \in \R_{+}$. We denote $\varepsilon_1,\ldots,\varepsilon_J$ the $J$ distinct values among $\left\{\abs{\coeff{\indcoef}}\right\}_{\indcoeff=1}^\dimcoeff$ and $d_1,\ldots,d_J$ their respective multiplicity orders with $\sum_{j=1}^J d_j=\dimcoeff$. Then the $\indcoef$-th component of $\boldsymbol{\rho}\triangleq\prox_{ g_1 }^{\delta} (\Vcoef )$, denoted as $\rho_\indcoef$, is given by
\begin{equation} \label{eq:prox_solution}
\rho_{\indcoef} =
\left\{
  \begin{array}{ll}
    \mathrm{sign}(\coef{\indcoef}) \phi_{\delta}(\Vcoeff) & \hbox{\emph{if}  $\vert \coef{\indcoef} \vert \geq \phi_{\delta}(\Vcoeff)$} \\
    \coef{\indcoef} & \hbox{\emph{otherwise}}
  \end{array}
\right.
\end{equation}
where
\begin{equation}
\label{eq:prox_phi}
\phi_{\delta}(\Vcoeff) = \max\left(0, \phi_{\delta}^1(\Vcoef),\ldots,\phi_{\delta}^J(\Vcoef)  \right)
\end{equation}
and, for $j=1,\ldots,J$,
\begin{equation}
\label{eq:prox_phij}
   \phi_{\delta}^j(\Vcoef) = \frac{1}{\sum_{k=1}^j d_k} \left(\sum_{k=1}^j d_k \varepsilon_k - \lambda\delta\right).
\end{equation}
\end{property}

\begin{proof}
\ifallannexe
  See Appendix \ref{app:proximity}.
\else
  See Appendix \cite[App. E]{Elvira2015TR}
\fi
\end{proof}

\begin{algorithm} \label{alg:prox}
\caption{Algorithm to compute $\prox_{ g_1 }^{\delta} (\Vcoef )$}
\KwIn{$\Vcoef$, $\delta$}
\BlankLine
Identify $\varepsilon_{1}\dots \varepsilon_{J}$ as the $J$ different values of $\abs{\coef{\indcoeff}}$'s, and $d_1 \dots d_J$ their respective multiplicity order \;
\For{$j  \gets 1$ \KwTo $J$} {
	Compute $\phi_{\delta}^j(\Vcoef)$ following \eqref{eq:prox_phij} \;
}
Compute $\phi_{\delta}(\Vcoef) = \max \left(0, \phi_{\delta}^1(\Vcoef),\ldots,\phi_{\delta}^J(\Vcoef) \right)$ \;
\For{$\indcoef \gets 1$ \KwTo $\dimcoef$} {
	\eIf{$\abs{\coef{\indcoef}} \geq \phi_{\delta}(\Vcoeff)$}{
		$\rho_{\indcoef} = \mathrm{sign}(\coef{\indcoef}) \phi_{\delta}(\Vcoeff)$ \;
   }{
   		$\rho_{\indcoef} = \coef{\indcoef}$ \;
  }
}
Set $\boldsymbol{\rho}=\left[\rho_1,\ldots,\rho_\dimcoeff\right]\transp$ \;
\BlankLine
\KwOut{$\boldsymbol{\rho} = \text{prox}_{ g_1 }^{\delta} (\Vcoef ) $}
\end{algorithm}
From Property \ref{prop:proximal}, under this proximity mapping, all components greater than a threshold are reduced to a common value, while all the others remain unchanged.

\begin{remark}
  As stated earlier, this proximity operator will be resorted while designing Monte Carlo algorithms able to generate random samples according to distributions derived from $g_1(\cdot)$. In this specific context, almost surely, the multiplicity orders $d_j$ ($j=1,\ldots,J$) are all equal to one, i.e., $J=N$ and $\varepsilon_\indcoeff=\abs{\coef{\indcoeff}}$ ($\indcoeff=1,\ldots,\dimcoef$).
\end{remark}

\subsection{Random variate generation}
	\label{subsec:random_variate_generation}

This paragraph introduces three distinct random variate generators that allow to draw samples according to the democratic distribution.

\subsubsection{Exact random variate generator}
\label{subsubsec:generator_exact}

Property \ref{prop:dominant_proba} combined with Lemma \ref{lem:dominant_pdf} permits to rewrite the joint distribution of a democratically distributed vector according to the following chain rule
\begin{align*}
p(\Vcoeff) =& \sum_{\indcoef=1}^{\dimcoef} p\left(\Vcoeff_{\backslash \indcoef}|\coef{\indcoef},\Vcoeff \in \cone_{\indcoeff}\right) p\left(\coeff{\indcoeff}|\Vcoeff \in \cone_{\indcoeff}\right) \mathrm{P}\left[\Vcoeff \in \cone_{\indcoeff}\right]\\
=& \sum_{\indcoef=1}^{\dimcoef} \left[\prod_{j\neq n} p\left(\coeff{j}|\coeff{\indcoef}, \Vcoeff \in \cone_{\indcoeff}\right)\right] p\left(\coeff{\indcoeff}|\Vcoeff \in \cone_{\indcoeff}\right) \mathrm{P}\left[\Vcoeff \in \cone_{\indcoeff}\right]
\end{align*}
where $\mathrm{P}\left[\Vcoeff \in \cone_{\indcoeff}\right]$, $ p\left(\coeff{\indcoeff}|\Vcoeff \in \cone_{\indcoeff}\right)$ and  $p\left(\coeff{j}|\coeff{\indcoef}, \Vcoeff \in \cone_{\indcoeff}\right)$ are given in  \eqref{eq:dominant_proba}, \eqref{eq:dominant_pdf} and \eqref{eq:dominant_proba_cond}, respectively. This finding can be fully exploited to design an efficient and exact random variate generator for the democratic distribution, see Algo.~\ref{alg:exact_sampling_prior}.

\begin{algorithm} \label{alg:exact_sampling_prior}
\caption{Democratic random variate generator using an exact sampling scheme.}
\KwIn{Parameter $\lambda$, dimension $\dimcoef$}
\BlankLine
\emph{\% Drawing the cone of the dominant component}\\
Sample $\indcoef_{\mathrm{dom}}$ uniformly on the set $\{1 \dots \dimcoef \}$\;
\emph{\% Drawing the dominant component}\\
Sample $\coef{\indcoef_{\mathrm{dom}}}$ according to \eqref{eq:dominant_pdf}\;
\emph{\% Drawing the non-dominant components}\\
\For{$j \gets 1$ \KwTo $\dimcoef$ ($j \neq \indcoef_{\mathrm{dom}}$)}  {
	Sample $\coef{j}$ according to \eqref{eq:dominant_proba_cond}\;
}
\BlankLine
\KwOut{$\Vcoeff=\left[\coef{1},\ldots,\coef{\dimcoef}\right]^T \sim \calD_\dimcoef(\lambda)$}
\end{algorithm}

\subsubsection{Gibbs sampler-based random generator}
\label{subsubsec:generator_gibbs}

Property \ref{prop:conditional} can be exploited to design a democratic random variate generator through the use of a Gibbs sampling scheme. It consists of successively drawing the components $\coef{\indcoef}$ according to the conditional distributions \eqref{eq:conditional}, defined as the mixtures of uniform and truncated Laplacian distributions. After a given number $\Nbi$ of burn-in iterations, this generator, described in Algo.~\ref{alg:gibs_sampling_prior}, provides samples asymptotically distributed according to the democratic distribution $\calD_\dimcoef(\lambda)$.\\

\begin{algorithm} \label{alg:gibs_sampling_prior}
\caption{Democratic random variate generator using a Gibbs sampling scheme.}
\KwIn{Parameter $\lambda$, dimension $\dimcoeff$, number of burn-in iterations $\Nbi$, total number of iterations $\Nmc$, initialization $\Vcoeff^{(1,0)}=\left[\coef{1}^{(1,0)},\ldots,\coef{\dimcoef}^{(1,0)}\right]^\Nmc$}
\BlankLine
\For{$t  \gets 1$ \KwTo $\Nmc$} {
	\For{$\indcoef \gets 1$ \KwTo $\dimcoef$} {
        Set $\Vcoef_{\backslash\indcoef}^{(t,n-1)}=\left[\coef{1}^{(t,n)},\ldots,\coef{n-1}^{(t,n)},\coef{n+1}^{(t,n-1)},\ldots,\coef{\dimcoef}^{(t,n-1)}\right]^T$\;
        Draw $\coef{\indcoef}^{(t,n)}$ according to \eqref{eq:conditional}\;
		 Set $\Vcoef^{(t,\indcoef)}=\left[\coef{1}^{(t,\indcoef)},\ldots,\coef{\indcoef-1}^{(t,\indcoef)},\coef{\indcoef}^{(t,\indcoef)},\coef{n+1}^{(t,\indcoef-1)},\ldots,\coef{\dimcoef}^{(t-1,\indcoef-1)}\right]^T$\;
	}
    Set $\Vcoef^{(t+1,0)}= \Vcoef^{(t,\dimcoef)}$\;
}
\BlankLine
\KwOut{$\Vcoef^{(t,0)} \sim \calD_\dimcoef(\lambda)$  (for $t>\Nbi$)}
\end{algorithm}

\subsubsection{P-MALA-based random generator}
\label{subsubsec:generator_pmala}
An alternative to draw samples according to the democratic distribution is the proximal Metropolis-adjusted Langevin algorithm (P-MALA) introduced in \cite{Pereyra2015prox}. P-MALA builds a Markov chain $\left\{\Vcoef^{(t)}\right\}_{t=1}^{\Nmc}$ whose stationary distribution is of the form
$$ p(\Vcoef) \propto \exp(- g(\Vcoef)) $$
where $g$ is a positive convex function with $\lim\limits_{\Vert \Vcoef \Vert \to \infty} g(\Vcoef) = + \infty$. It relies on successive Metropolis Hastings moves with Gaussian proposal distributions whose mean has been chosen as the proximal operator of $g(\cdot)$ evaluated at the current state of the chain. In the particular case of the democratic distribution, P-MALA can be implemented by exploiting the derivations in paragraph \ref{subsec:proximity_operator}, where $g\left(\cdot\right)=g_1\left(\cdot\right)$ has been defined in \eqref{eq:definition_g1}. More precisely, at iteration $t$ of the sampler, a candidate $\Vcoef^*$ is proposed as
\begin{equation} \label{eq:prox_candidate_prior}
	\Vcoef^* \sachant \Vcoef^{(t-1)} \sim \mathcal{N} \left( \text{prox}_{g_1}^{\delta/2} \left(\Vcoef^{(t-1)}\right), \delta \mathbf{I}_{\dimcoef} \right).
\end{equation}
Then this candidate is accepted as the new state $ \Vcoef^{(t)} $ with probability
\begin{equation} \label{eq:prior:prox_arate}
	\alpha = \min \left( 1, \frac{ p\left( \Vcoef^* \sachant \lambda \right) }{ p\left( \Vcoef^{(t-1)} \sachant \lambda \right)} \frac{q\left( \Vcoef^{(t-1)} \sachant \Vcoef^* \right)} { q\left( \Vcoef^* \sachant \Vcoef^{(t-1)} \right) } \right)
\end{equation}
where $q\left( \Vcoef^* \sachant \Vcoef^{(t-1)} \right)$ is the pdf of the Gaussian distribution given by \eqref{eq:prox_candidate_prior}. The algorithmic parameter $\delta$ is empirically chosen such that the acceptance rate of the sampler lies between $0.4$ and $0.6$. The full algorithmic scheme is available in Algo. \ref{alg:pmala_sampling_prior}.

\begin{algorithm} \label{alg:pmala_sampling_prior}
\caption{Democratic random variate generator using P-MALA.}
\KwIn{Parameter $\lambda$, dimension $\dimcoeff$, number of burn-in iterations $\Nbi$, total number of iterations $\Nmc$, algorithmic parameter $\delta$, initialization $\Vcoeff^{(0)}$}
\BlankLine
\For{$t  \gets 1$ \KwTo $\Nmc$} {
    Draw $\Vcoef^* \sachant \Vcoef^{(t-1)} \sim \mathcal{N} \left( \text{prox}_{g_1}^{\delta/2} \left(\Vcoef^{(t-1)}\right), \delta \mathbf{I}_{\dimcoef} \right)$ \;
    Compute $\alpha$ following \eqref{eq:prior:prox_arate} \;
    Draw $w \sim \calU(0,1)$ \;
    \eIf{$w< \alpha$}{
		Set $\Vcoef^{(t)}=\Vcoef^*$ \;
   }{
   		Set $\Vcoef^{(t)}=\Vcoef^{(t-1)}$ \;
    }
    }
\BlankLine
\KwOut{$\Vcoef^{(t)} \sim \calD_\dimcoef(\lambda)$  (for $t>\Nbi$)}
\end{algorithm}

\subsubsection{Random generator performance comparison}
\label{subsubsec:generator_performance}

\begin{figure}[h!]
\centering
	\scalebox{.24}{
%
%
\definecolor{mycolor1}{RGB}{56,108,176}
\begin{tikzpicture}

\begin{axis}[%
width=6.027778in,
height=4.754167in,
at={(0, 0)},
scale only axis,
area legend,
xmin=0.5,
xmax=15.5,
xtick={ 1,  2,  3,  4,  5,  6,  7,  8,  9, 10, 11, 12, 13, 14, 15},
xticklabels={ 1,  2,  3,  4,  5,  6,  7,  8,  9, 10, 11, 12, 13, 14, 15},
ticklabel style={font=\Huge\sffamily},
ymin=0,
ymax=0.14,
ylabel=Empirical Autocorrelation,
ylabel style={at={(-0.13, 0.5)}, font=\Huge\sffamily},
]

\addplot[ybar,bar width=0.321481in,draw=black,fill=mycolor1] plot table[row sep=crcr] {%
1	0.12470809920181\\
2	0.0505082103502347\\
3	0.0472714945245232\\
4	0.0327649261140032\\
5	0.0275034106077464\\
6	0.0266025361293348\\
7	0.0259294848646049\\
8	0.0229649752062993\\
9	0.0217624003299555\\
10	0.0208727352515409\\
11	0.014959553247615\\
12	0.00888951007286367\\
13	0.00558962936541676\\
14	0.00386300420841023\\
15	0.00198398485511713\\
};
\end{axis}
\end{tikzpicture}%
%
%
\definecolor{mycolor1}{RGB}{56,108,176}
\begin{tikzpicture}

\begin{axis}[%
width=6.027778in,
height=4.754167in,
at={(0, 0)},
scale only axis,
area legend,
xmin=0.5,
xmax=15.5,
xtick={ 1,  2,  3,  4,  5,  6,  7,  8,  9, 10, 11, 12, 13, 14, 15},
ytick={0, 0.02, 0.04,  0.06,  0.08, 0.1, 0.12, 0.14},
ymin=0,
ymax=0.14,
ticklabel style={font=\Huge\sffamily},
]
\addplot[ybar,bar width=0.321481in,draw=black,fill=mycolor1] plot table[row sep=crcr] {%
1	0.0740446024138876\\
2	0.0658052728730348\\
3	0.0643304838210699\\
4	0.0449466847062629\\
5	0.0431893279579161\\
6	0.0428082083163659\\
7	0.0380277900405265\\
8	0.0271144798284669\\
9	0.0256017097037611\\
10	0.0171013262034661\\
11	0.0114187197729156\\
12	0.0113620451992157\\
13	0.0112048634748192\\
14	0.0042088949329269\\
15	0.000104045014103701\\
};
\end{axis}
\end{tikzpicture}%
 } \\
	\scalebox{.24}{
%
%
\definecolor{mycolor1}{RGB}{56,108,176}
\begin{tikzpicture}

\begin{axis}[%
width=6.027778in,
height=4.754167in,
at={(0, 0)},
scale only axis,
area legend,
xmin=0.5,
xmax=15.5,
xtick={ 1,  2,  3,  4,  5,  6,  7,  8,  9, 10, 11, 12, 13, 14, 15},
ymin=0,
ymax=0.14,
ylabel=Empirical Autocorrelation,
ylabel style={at={(-0.13, 0.5)}, font=\Huge\sffamily},
ticklabel style={font=\Huge\sffamily},
]
\addplot[ybar,bar width=0.321481in,draw=black,fill=mycolor1] plot table[row sep=crcr] {%
1	0.129341898932363\\
2	0.0890659305379718\\
3	0.0783248889530888\\
4	0.0650587967537061\\
5	0.0623416340528974\\
6	0.0571726304031337\\
7	0.0533348243776335\\
8	0.0495237988179056\\
9	0.0474608359770518\\
10	0.04664232995996\\
11	0.0391897581552833\\
12	0.0324774890291238\\
13	0.0134565534872958\\
14	0.0131884939681456\\
15	0.000832256784912182\\
};
\end{axis}
\end{tikzpicture}%
%
%
\definecolor{mycolor1}{RGB}{56,108,176}
\begin{tikzpicture}

\begin{axis}[%
width=6.027778in,
height=4.754167in,
at={(0, 0)},
scale only axis,
area legend,
xmin=0.5,
xmax=15.5,
xtick={ 1,  2,  3,  4,  5,  6,  7,  8,  9, 10, 11, 12, 13, 14, 15},
ytick={0, 0.02, 0.04,  0.06,  0.08, 0.1, 0.12, 0.14},
ymin=0,
ymax=0.14,
ticklabel style={font=\Huge\sffamily},
]
\addplot[ybar,bar width=0.321481in,draw=black,fill=mycolor1] plot table[row sep=crcr] {%
1	0.0785388650170329\\
2	0.0646132688742219\\
3	0.061339173722672\\
4	0.0564773510181939\\
5	0.0516152783889759\\
6	0.0451533824340402\\
7	0.0436659052432667\\
8	0.0415735794122206\\
9	0.0361634621000179\\
10	0.0264706192036673\\
11	0.0189760883106732\\
12	0.0126124530869837\\
13	0.00560739903162634\\
14	0.00245761170045515\\
15	0.00229098241419281\\
};
\end{axis}
\end{tikzpicture}%
 } \\
	\scalebox{.24}{
%
%
\definecolor{mycolor1}{RGB}{56,108,176}
\begin{tikzpicture}

\begin{axis}[%
width=6.027778in,
height=4.754167in,
at={(0, 0)},
scale only axis,
area legend,
xmin=0.5,
xmax=15.5,
xtick= {1,  2,  3,  4,  5,  6,  7,  8,  9, 10, 11, 12, 13, 14, 15},
ytick = {0, 0.1, 0.2, 0.3, 0.4, 0.5, 0.6, 0.7},
xlabel= Lag,
xlabel style={font=\Huge\sffamily},
ymin=0,
ymax=0.7,
ylabel=Empirical Autocorrelation,
ylabel style={at={(-0.13, 0.5)}, font=\Huge\sffamily},
ticklabel style={font=\Huge\sffamily}
]
\addplot[ybar,bar width=0.321481in,draw=black,fill=mycolor1] plot table[row sep=crcr] {%
1	0.583410288399931\\
2	0.320920600130457\\
3	0.177888275725154\\
4	0.10631441427074\\
5	0.0941664720440852\\
6	0.0808810253078735\\
7	0.0728680384934901\\
8	0.0707757461279154\\
9	0.0643384712101952\\
10	0.0518774974179639\\
11	0.0288736147461211\\
12	0.00788911919252772\\
13	0.0054989711029341\\
14	0.00510443593135421\\
15	0.00338891680131549\\
};
\end{axis}
\end{tikzpicture}%
%
%
\definecolor{mycolor1}{RGB}{56,108,176}
\begin{tikzpicture}

\begin{axis}[%
width=6.027778in,
height=4.754167in,
at={(0, 0)},
scale only axis,
area legend,
xmin=0.5,
xmax=15.5,
xtick={ 1,  2,  3,  4,  5,  6,  7,  8,  9, 10, 11, 12, 13, 14, 15},
ytick = {-0.01, 0, 0.1, 0.2, 0.3, 0.4, 0.5, 0.6, 0.7},
xlabel=Lag,
xlabel style={font=\Huge\sffamily},
ymin=0,
ymax=0.7,
ylabel ={\empty},
ylabel style={at={(-0.10, 0.5)}, font=\Huge\sffamily},
ticklabel style={font=\Huge\sffamily},
]
\addplot[ybar,bar width=0.321481in,draw=black,fill=mycolor1] plot table[row sep=crcr] {%
1	0.550283401144439\\
2	0.533571506956361\\
3	0.514144622836426\\
4	0.49890069081937\\
5	0.48365040274257\\
6	0.467481923707543\\
7	0.454798991838752\\
8	0.442177365563115\\
9	0.429909404333765\\
10	0.417781013277364\\
11	0.409194128817572\\
12	0.401579598886037\\
13	0.393053675694079\\
14	0.385356887005425\\
15	0.377537859158236\\
};
\end{axis}
\end{tikzpicture}%
 }
	\caption{First $15$ lags of the empirical autocorrelation function when generating $500$ samples drawn from $\calD_{\dimcoeff}(\lambda)$ using the exact (top), Gibbs sampler-based (middle) and P-MALA (bottom) variate generators with $\lambda=3$ for $\dimcoef=2$ (left) and $\dimcoef=50$ (right).} \label{fig:random_variate_generator:autocorrelation}
\end{figure}
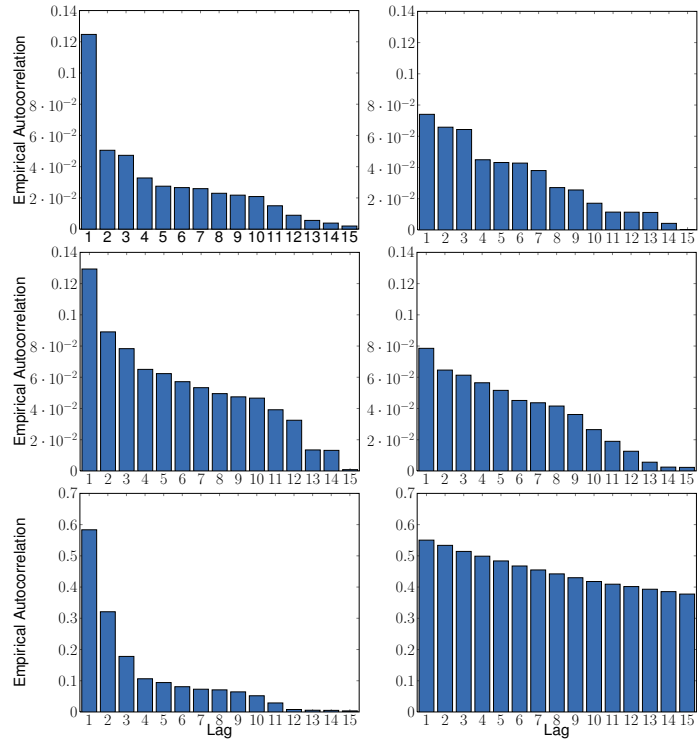

\begin{table}
	\caption{Cumulative times to draw $500$ samples according the democratic distribution $\calD_{\dimcoef}\left(\lambda\right)$ with $\lambda = 3$ for various $\dimcoef$.}
	\label{tab:time_sampling}
	\centering
    \begin{tabular}{ccccc}
  		\toprule[1.5pt]
		$\dimcoef$    & Exact (ms) & Gibbs (ms)  & P-MALA (ms)   \\
		\midrule
		$2 $   & $2.06 \times 10^{-1}$ & $2.05 \times 10^3$  & $1.11 \times 10^2$\\
		$5 $   & $3.30 \times 10^{-1}$ & $2.66 \times 10^3$  & $1.12 \times 10^2$\\
		$10$   & $3.46 \times 10^{-1}$ & $3.60 \times 10^3$  & $1.14 \times 10^2$\\
		$50$   & $9.24 \times 10^{-1}$ & $1.01 \times 10^4$  & $1.27 \times 10^2$\\
		$100$  & $1.71 \times 10^{ 0}$ & $1.85 \times 10^4$  & $1.46 \times 10^2$\\
		\bottomrule[1.5pt]
    \end{tabular}
\end{table}

Figure~\ref{fig:random_variate_generator:autocorrelation} compares the first $15$ lags of the empirical autocorrelation function (ACF), computed with $500$ samples drawn from the democratic distribution $\calD_{\dimcoeff}(3)$ for $\dimcoef=2$ (left) and $\dimcoef=50$ (right) using the exact (top), Gibbs sampler-based (middle) and P-MALA (bottom) variate generators. In lower dimensional cases, the chain generated with exact sampling has remarkably lower autocorrelation.

Computational times required to generate $1000$ samples from the democratic distribution for various dimensions are reported in Table~\ref{tab:time_sampling}. These results show that the Gibbs sampler-based method has a significantly higher cost when compared with the exact random generator. Whereas the exact sampler easily scales in higher dimension, the Gibbs based sampler needs approximately $10^4$ more time. This is explained by its intrinsic algorithmic structure: the Gibbs sampler-based method requires to draw a multinomial variable for each component, followed by either exponential or uniform distributed variables. Conversely, exact random generator only needs to generate one gamma distributed variable and $(\dimcoef-1)$ uniform samples.

From these findings, the Gibbs sampler-based and P-MALA strategies might seem out of interest since significantly outperformed by the exact random generator in terms of time computation and mixing performance. However, both exhibit interesting properties that can be exploited in a more general scheme. First, the Gibbs sampler-based generator shows that each component of a democratically distributed vector can be easily generated conditionally on the others. Then, P-MALA exploits the algorithmic derivation of the proximity operator associated with $g_1(\cdot)$ to draw vectors asymptotically distributed according to the democratic distribution. This opens the door to extended schemes for sampling according to a posterior distribution resulting from a democratic prior when possibly no exact sampler is available. This will be discussed in Section~\ref{sec:inverse_pb}.

\section{Democratic prior in a linear inverse problem}
\label{sec:inverse_pb}
This section aims to provide a Bayesian formulation of the model underlying the problem described by~\eqref{eq:problem}. From a Bayesian perspective, the solution of~\eqref{eq:problem} can be straightforwardly interpreted as the MAP estimator associated with a linear observation model characterized by an additive Gaussian residual and complemented by a democratic prior assumption. Assuming a Gaussian residual results in a quadratic discrepancy measure and, as a consequence, in a quadratic data fidelity term as in~\eqref{eq:problem}. Setting the anti-sparse coding problem into a fully Bayesian framework paves the way to a comprehensive statistical description of the solution. Moreover, it permits to implement other algorithmic strategies beyond MAP estimation, namely Markov chain Monte Carlo techniques.

\subsection{Hierarchical Bayesian model} \label{subsec:bayesian_model}

Let $\Vobs = \left[ \obs{1} \dots \obs{\dimobs} \right]\transp$ denote an observed measurement vector. These observations are assumed to be related to an unknown description vector $\Vcoef = \left[ \coef{1} \dots \coef{\dimcoef} \right]\transp$ through a known coding matrix $\MATop$ according to the linear model
\begin{equation} \label{eq:model}
	\Vobs  = \MATop \Vcoef  + \Vnoise.
\end{equation}
The residual vector $\Vnoise = \left[ \noise{1} \dots \noise{\dimcoef} \right]\transp$ is assumed to be distributed according to a centered multivariate Gaussian distribution $\mathcal{N}(\boldsymbol{0}_{\dimobs}, \noisevar \bsI_M)$. The Bayesian model is introduced in what follows. It relies on the definition of the likelihood function associated with the observation vector $\Vobs$ and on the choice of prior distributions for the unknown parameters, i.e., the representation vector $\Vcoef$ and the residual variance $\noisevar$, assumed to be a priori independent.

\subsubsection{Likelihood function}
The Gaussian property of the additive residual term yields the following likelihood function
\begin{equation} \label{eq:likelihood}
f(\Vobs | \Vcoef, \noisevar ) = \left( \frac{1}{2\pi \noisevar} \right)^{ \frac{\dimobs}{2} } \exp \left[ - \frac{1}{2 \noisevar} \normq{\Vobs - \MATop \Vcoef } \right].
\end{equation}

\subsubsection{Residual variance prior}

A noninformative Jeffreys prior distribution is chosen for the residual variance $\noisevar$
\begin{equation}
\label{eq:prior_noise}
	f\left(\noisevar\right) \propto \frac{1}{\noisevar}.
\end{equation}

\subsubsection{Description vector prior}

As motivated earlier, the democratic distribution introduced in Section \ref{sec:democratic} is assigned as the prior distribution of the $\dimcoef$-dimensional unknown vector $\Vcoef$
\begin{equation}
\label{eq:prior_representation}
\Vcoef \mid \demoparamini \sim \calD_\dimcoef(\demoparamini).
\end{equation}
In the following, the hyperparameter $\demoparamini$ is set as $\demoparamini=\dimcoef\demoparam$, where $\demoparam$ is assumed to be unknown. Enforcing the parameter of the democratic distribution to depend on the problem dimension allows the prior to be scaled with this dimension. Indeed, as stated in \eqref{eq:dominant_pdf}, the absolute value of the dominant component is distributed according to the Gamma distribution  $\calG\left(\dimcoef,\demoparamini\right)$, whose mean and variance are ${\dimcoef}\slash{\demoparamini}$ and  ${\dimcoef}\slash{\demoparamini^2}$, respectively. With the proposed scalability, the prior mean is constant with the dimension
\begin{equation}
	\mathrm{E}\left[|\coeff{\indcoeff}|\mid\Vcoeff \in \cone_{\indcoeff},\demoparam\right]={1}\slash{\demoparam}
\end{equation}
and the variance tends to zero
\begin{equation}
	\mathrm{var}\left[|\coeff{\indcoeff}|\mid\Vcoeff \in \cone_{\indcoeff},\demoparam\right]= {1}\slash{(\dimcoef\demoparam^2)}.
\end{equation}

\subsubsection{Hyperparameter prior}

The prior modeling introduced in the previous paragraph is complemented by assigning prior distribution to the unknown hyperparameter $\demoparam$, introducing a second level in the Bayesian hierarchy. More precisely, a conjugate Gamma distribution is chosen as a prior for $\demoparam$
\begin{equation}
\label{eq:prior_hyperparameter}
\demoparam \sim \mathcal{G}(a, b).
\end{equation}
since the conjugacy property allows the posterior distribution to be easily derived. The values of $a$ and $b$ will be chosen to obtain a flat prior.

\subsubsection{Posterior distribution}

The posterior distribution of the unknown parameter vector $\paramvect = \{ \Vcoeff,\noisevar,\demoparam \}$ can be computed from the following hierarchical structure:
\begin{equation} \label{eq:joint_short}
f(\paramvect \sachant \Vobs) \propto f(\Vobs \sachant \Vcoeff,\noisevar) f(\Vcoeff,\noisevar \sachant \demoparam)  f(\demoparam)
\end{equation}
where
\begin{equation}
f(\Vcoeff,\noisevar \sachant \demoparam)= f(\noisevar )  f(\Vcoef \sachant \demoparam)
\end{equation}
and $f(\Vobs \sachant \Vcoeff,\noisevar)$, $f(\noisevar )$, $f(\Vcoef \sachant \demoparam)$ and $f(\demoparam)$ have been defined in Eq. \eqref{eq:likelihood} to \eqref{eq:prior_hyperparameter}, respectively. Thus, this posterior distribution can be written as
\begin{eqnarray} \label{eq:joint_complete}
f(\Vcoef, \noisevar, \demoparam \sachant \Vobs) &\propto&   \exp \left( - \frac{1}{2\noisevar} \normq{\Vobs - \MATop \Vcoef} \right) \nonumber \\
&\times& \frac{1}{C_{\dimcoef}(\demoparam \dimcoef)} \exp \left( - \demoparam \dimcoef \norminf{\Vcoef} \right) \nonumber \\
& \times & \left(\frac{1}{\noisevar}\right)^{\frac{\dimobs}{2}+1} \boldsymbol{1}_{\R_+}(\noisevar) \nonumber\\
&\times& \frac{b^a}{\Gamma(b)} \demoparam^{a-1} \exp \left(- b \demoparam \right).
\end{eqnarray}
As expected, for given values of the residual variance $\noisevar$ and the democratic parameter $\demoparamini=\demoparam\dimcoef$, maximizing the posterior \eqref{eq:joint_complete} can be formulated as the optimization problem in \eqref{eq:problem}, for which some algorithmic strategies have been for instance introduced in \cite{Studer2015sub} and \cite{Fuchs2011asilomar}. In this paper, a different route has been taken by deriving inference schemes relying on MCMC algorithms. This choice permits to include the nuisance parameters $\noisevar$ and $\demoparam$ into the model and to estimate them jointly with the representation vector $\Vcoeff$. Moreover, since the proposed MCMC algorithms generate a collection $\left\{ \left( \Vcoef^{(t)}, \demoparam^{(t)}, \sigma^{2(t)} \right) \right\}_{t=1}^{N_{\mathrm{MC}}}$ asymptotically distributed according to the posterior of interest \eqref{eq:joint_short}, they provide a good knowledge of the statistical distribution of the solutions.

\subsection{MCMC algorithm}
\label{subsec:MCMC}
This section introduces two MCMC algorithms to generate samples according to the posterior \eqref{eq:joint_complete}. They are two specific instances of Gibbs samplers which generate samples according to the conditional distributions associated with the posterior \eqref{eq:joint_complete}, following Algo. \ref{alg:posterior_sampling}. As shown below, the steps for sampling according to the conditional distributions of the residual variance $f(\noisevar |\Vobs,\Vcoeff)$ and the democratic parameter $f\left(\demoparam|\Vcoeff\right)$ are straightforward. In addition, generating samples for the representation vector $f(\Vcoeff|\demoparam,\Vobs)$ can be achieved component-by-component using $\dimcoeff$ Gibbs moves, following the strategy in paragraph \ref{subsubsec:generator_gibbs}. However, for high dimensional problems, such a crude strategy may suffer from poor mixing properties, leading to slow convergence of the algorithm. To alleviate this issue, an alternative approach consists of sampling the full vector $\Vcoeff|\demoparam,\Vobs$ using a P-MALA step \cite{Pereyra2015prox}, similar to the one proposed in paragraph \ref{subsubsec:generator_pmala}. These two strategies are detailed in the following paragraphs.

\begin{algorithm} \label{alg:posterior_sampling}
\caption{Gibbs sampler}
\KwIn{Observation vector $\Vobs$, coding matrix $\MATop$, hyperparameters $a$ and $b$, number of burn-in iterations $\Nbi$, total number of iterations $\Nmc$, algorithmic parameter $\delta$, initialization $\Vcoeff^{(0)}$}
\BlankLine
\For{$t  \gets 1$ \KwTo $\Nmc$} {
    \emph{\% Drawing the residual variance}\\
	Sample $\noisesymb^{2(t)}$ according to \eqref{eq:hierarchical_gibbs:posterior_noisevar}. \;
    \emph{\% Drawing the democratic parameter}\\
	Sample $\demoparam^{(t)}$ according to \eqref{eq:hierarchical_gibbs:posterior_lambda}. \;
    \emph{\% Drawing the representation vector}\\
	Sample $\Vcoeff^{(t)}$ using, either (see paragraph \ref{subsubsec:gibbs_description_vector}) \label{alg:step_description_vector}
    \begin{itemize}
      \item Gibbs steps, i.e., following \eqref{eq:hierarchical_gibbs:posterior_coef_gibbs} \;
        \item P-MALA step, i.e., following \eqref{eq:prox_candidate} and \eqref{eq:prox_arate}\;
    \end{itemize}
}
\BlankLine
\KwOut{A collection of samples $\left\{\demoparam^{(t)},\noisesymb^{2(t)},\Vcoeff^{(t)}\right\}_{t=\Nbi+1}^{\Nmc}$ asymptotically distributed according to \eqref{eq:joint_complete}.}
\end{algorithm}

\subsubsection{Sampling the residual variance} Sampling according to the conditional distribution of the residual variance can be conducted according to the following inverse-gamma distribution
\begin{equation} \label{eq:hierarchical_gibbs:posterior_noisevar}
\noisevar \sachant \Vobs, \Vcoef \sim \mathcal{IG} \left( \frac{\dimobs}{2}, \frac{1}{2} \normq{ \Vobs - \MATop \Vcoef } \right)
\end{equation}

\subsubsection{Sampling the democratic hyperparameter} Looking carefully at \eqref{eq:joint_complete}, the conditional posterior distribution of the democratic parameter $\demoparam$ is
\begin{equation}
f(\demoparam \sachant \Vcoef) \propto \demoparam^{\dimcoef} \exp\left( - \demoparam \norminf{\Vcoef} \right) \demoparam^{a-1} \exp \left( - b \demoparam \right).
\end{equation}
Therefore, sampling according to $f(\demoparam \sachant \Vcoef)$ is achieved as follows
\begin{equation} \label{eq:hierarchical_gibbs:posterior_lambda}
\demoparam \sachant \Vcoeff \sim \mathcal{G}(a + \dimcoef, b+\dimcoef \norminf{\Vcoef})
\end{equation}

\subsubsection{Sampling the description vector}
\label{subsubsec:gibbs_description_vector}
Following the technical developments of paragraph \ref{subsec:random_variate_generation}, two strategies can be considered to generate samples according to the conditional posterior distribution of the representation vector $f(\Vcoeff|\demoparam,\noisevar,\Vobs)$. They are detailed below.

\paragraph{Component-wise Gibbs sampling}
\label{subsubsec:gibbs_description_vector_Gibbs}
A first possibility to draw a vector $\Vcoeff$ according to $f(\Vcoeff|\demoparam,\noisevar,\Vobs)$ is to successively sample according to the conditional distribution of each component given the others, namely, $f(\coeff{\indcoeff}|\Vcoeff_{\backslash \indcoeff},\demoparam,\noisevar, \Vobs)$, as in algorithm of paragraph \ref{subsubsec:generator_gibbs}. More precisely, straightforward computations yield the following $3$-mixture of truncated Gaussian distributions for this conditional
\begin{equation} \label{eq:hierarchical_gibbs:posterior_coef_gibbs}
		\coef{\indcoef} \sachant \Vcoef_{\backslash \indcoef}, \demoparam,  \noisevar , \Vobs
        \sim \sum_{i=1}^3 \omega_{i\indcoef} \mathcal{N}_{\calI_{i\indcoef}}\left( \mu_{i\indcoef}, s^2_{\indcoef} \right)
\end{equation}
where $\mathcal{N}_{\calI}(\cdot,\cdot)$ denotes the Gaussian distribution truncated on the $\calI$ and the truncation sets are defined as
\begin{eqnarray*}
  \calI_{1n} &= \left( -\infty, -\norminf{\Vcoef_{\backslash \indcoef}} \right)\\
  \calI_{2n} &= \left( -\norminf{\Vcoef_{\backslash \indcoef}}, \norminf{\Vcoef_{\backslash \indcoef}} \right)\\
  \calI_{3n} &= \left( \norminf{\Vcoef_{\backslash \indcoef}}, +\infty \right).
\end{eqnarray*}
The probabilities $\omega_{i\indcoef}$ ($i=1,2,3$) as well as the (hidden) means $\mu_{i\indcoef}$ ($i=1,2,3$) and variance $s^2_{\indcoef}$ of these truncated Gaussian distributions are given in Appendix. This specific nature of the conditional distribution is intrinsically related to the nature of the conditional prior distribution stated in Property \ref{prop:conditional}, which has already exhibited a $3$-component mixture: one uniform distribution and two (shifted) exponential distributions defined over $\calI_{2n}$, $\calI_{1n}$ and $\calI_{3n}$, respectively (see Remark \ref{rem:conditional}). Note that sampling according to truncated distributions can be achieved using the strategy proposed in \cite{Chopin2011}.

\paragraph{P-MALA}
\label{subsubsec:gibbs_description_vector_PMALA}
Similarly to the strategy developed in paragraph \ref{subsubsec:generator_pmala} to sample according to the prior distribution, sampling according to the conditional distribution $f(\Vcoeff|\demoparam,\noisevar,\Vobs)$ can be achieved using a P-MALA step \cite{Pereyra2015prox}. In this case, the distribution of interest can be written as
\begin{equation*}
  f(\Vcoeff|\demoparam,\noisevar,\Vobs) \propto \exp\left(g\left(\Vcoeff\right)\right)
\end{equation*}
where $g\left(\Vcoeff\right)$ derives from the Gaussian (negative log-) likelihood function and the (negative log-) distribution of the democratic prior so that
\begin{equation} \label{eq:hierarchical_Bayesian:prox:g}
	 g(\Vcoef) = \frac{1}{2\noisevar} \normq{\Vobs - \MATop \Vcoef} + \lambda \norminf{\Vcoef}
\end{equation}
with $\lambda = \demoparam\dimcoeff$. However, up to the authors' knowledge, the proximal operator associated with $g(\cdot)$ in \eqref{eq:hierarchical_Bayesian:prox:g} has no close form solution. To alleviate this problem, a first order approximation is considered\footnote{Note that a similar step is involved in the fast iterative truncation algorithm (FITRA) \cite{Studer2013jsac}, a deterministic counterpart of the proposed algorithm and considered in the next section for comparison.}, as recommended in \cite{Pereyra2015prox}
\begin{equation}
	\text{prox}_g^{\delta/2} (\Vcoef) \approx \text{prox}_{g_1}^{\delta/2} \left( \Vcoef + \delta \; \nabla \left[ \frac{1}{2\noisevar}\normq{\Vobs-\MATop\Vcoef} \right] \right)
\end{equation}
where $g_1(\cdot) = \lambda \norminf{\cdot}$ has been defined in paragraph \ref{subsubsec:generator_pmala} and the corresponding proximity mapping is available through Algo.~\ref{alg:prox}. Finally, at iteration $t$ of the main algorithm, sampling according to the conditional distribution $f(\Vcoeff|\demoparam,\noisevar,\Vobs)$ consists of drawing a candidate
\begin{equation} \label{eq:prox_candidate}
	\Vcoef^* \sachant \Vcoef^{(t-1)} \sim \mathcal{N} \left( \text{prox}_{g}^{\delta/2} \left(\Vcoef^{(t-1)}\right), \delta \mathbf{I}_{\dimcoef} \right)
\end{equation}
and to accept this candidate as the new state $ \Vcoef^{(t)} $ with probability
\begin{equation} \label{eq:prox_arate}
	\alpha = \min \left( 1, \frac{ f\left( \Vcoef^* \sachant \demoparam,\noisevar,\Vobs \right) }{ f\left( \Vcoef^{(t-1)} \sachant \demoparam,\noisevar,\Vobs \right)} \frac{q\left( \Vcoef^{(t-1)} \sachant \Vcoef^* \right)} { q\left( \Vcoef^* \sachant \Vcoef^{(t-1)} \right) } \right).
\end{equation}

\subsection{Inference} \label{subsec:inference}

The sequences $\left\{ \Vcoef^{(t)}, \sigma^{2(t)}, \mu^{(t)} \right\}_{t=1}^{\Nmc}$ generated by the MCMC algorithms proposed in paragraph~\ref{subsec:MCMC} are used to approximate Bayesian estimators. After a burn-in period of $N_{\mathrm{bi}}$ iterations, the set of generated samples $\calX = \left\{ \Vcoef^{(t)}\right\}_{t=\Nbi+1}^{\Nmc}$ is asymptotically distributed according to the marginal posterior distribution $f\left(\Vcoeff|\Vobs\right)$, resulting from the marginalization of the joint posterior distribution $f\left(\Vcoeff,\noisevar,\demoparam|\Vobs\right)$ in \eqref{eq:joint_complete} over the nuisance parameters $\noisevar$ and $\demoparam$
\begin{eqnarray}
  f\left(\Vcoeff|\Vobs\right) &=  & {\int} f\left(\Vcoeff,\noisevar,\demoparam|\Vobs\right)d\noisevar d\demoparam \\
            &\propto &\left\|\Vobs - \MATop \Vcoef\right\|_2^{-\frac{\dimobs}{2}} \left(b+\dimcoef\norminf{\Vcoeff}\right)^{-(a+\dimcoef)}.
\end{eqnarray}
As a consequence, while the minimum mean square error (MMSE) estimator of the representation vector $\Vcoeff$ can be approximated as an empirical average over the set $\calX$
\begin{eqnarray}
  \MMSE{\Vcoeff} &=& \mathrm{E}\left[\Vcoeff|\Vobs\right] \\
                & \simeq & \frac{1}{\Nr} \sum_{t=1}^{\Nmc} \Vcoeff^{(t)}. \label{eq:MMSE}
\end{eqnarray}
with $\Nr = \Nmc-\Nbi$, the marginal maximum a posteriori (mMAP) estimator can be approximated as
\begin{eqnarray}
  \MAPm{\Vcoeff} &=& \operatornamewithlimits{argmax}_{\Vcoeff\in{\mathbb{R}}^{\dimcoeff}} f\left(\Vcoeff|\Vobs\right)\\
                & \simeq & \operatornamewithlimits{argmax}_{\Vcoeff^{(t)} \in \calX} f\left(\Vcoeff^{(t)}|\Vobs\right). \label{eq:mMAP}
\end{eqnarray}

\section{Experiments}
\label{sec:experiments}
This section reports several simulation results to illustrate the performance of the Bayesian anti-sparse coding algorithms introduced in Section~\ref{sec:inverse_pb}. In paragraph~\ref{subsec:geweke_experiment}, the validity of the samplers derived in paragraph \ref{subsec:MCMC} has been assessed following the experimental scheme in \cite{g04}, exploiting some properties of the democratic distributions enounced in Section~\ref{sec:democratic}. Paragraph \ref{subsec:toy_example} evaluates the performances of the two versions of the samplers (i.e., using Gibbs or P-MALA steps) on a toy example, by considering measurements resulting from a representation vector whose coefficients are all equal up to a sign. Finally, paragraph \ref{subsec:perf_comparison} compares the performances of the proposed algorithm and its deterministic counterpart introduced in \cite{Studer2013jsac}. For all experiments, the coding matrices $\MATop$ have been chosen as randomly subsampled columnwise DCT matrices since they have shown to yield democratic representations with small $\ell_{\infty}$-norm and good democracy bounds \cite{Studer2015sub}. However, note that a deep investigation of these bounds is out of the scope of the present paper.

\subsection{Assessment of the Bayesian anti-sparse coding algorithms} \label{subsec:geweke_experiment}

We first consider a first experiment to assess the validity of the Monte Carlo algorithms introduced in paragraph \ref{subsec:MCMC} and, in particular, the two methods proposed to sample according to the conditional distribution of the representation vector $f\left(\Vcoeff|\Vobs,\noisevar, \demoparam\right)$, see step \ref{alg:step_description_vector} in Algo. \ref{alg:posterior_sampling}. To this aim, the so-called \emph{successive conditional sampling} strategy proposed by Geweke in \cite{g04} is followed for fixed nuisance parameters $\noisevar$ and $\demoparam$. This procedure does not fully assert the correctness of the sampler but it may help to detect errors, e.g., as in \cite{knowles2011}. More precisely, it consists of drawing a sequence $\left\{\Vcoef^{(t)}, \Vobs^{(t)}\right\}_{t=1}^{\Nmc}$ asymptotically distributed according to the joint distribution $f\left(\Vcoeff,\Vobs|\noisevar, \demoparam\right)$ using the Gibbs sampler described in Algo.~\ref{alg:gewekeExp}.

\begin{algorithm} \label{alg:gewekeExp}
\caption{Successive conditional sampling}
\KwIn{Residual variance $\noisevar$, democratic parameter $\mu$, coding matrix $\MATop$.}
\BlankLine
Sample $\Vcoeff^{(0)}$ according to $\calD_\dimcoef(\mu \dimcoef)$ \;
\For{$t \gets 1$ \KwTo $\Nmc$}  {
	Sample $\Vobs^{(t)} \sachant \Vcoef^{(t-1)}, \noisevar \sim \mathcal{N}(\MATop \; \Vcoef^{(t)}, \noisevar)$ \;
	Sample $\Vcoef^{(t)} \sachant \Vobs^{(t)}, \mu, \noisevar$ using, either
\begin{itemize}
  \item Gibbs steps, i.e., following \eqref{eq:hierarchical_gibbs:posterior_coef_gibbs} \;
  \item P-MALA step, i.e., following \eqref{eq:prox_candidate} and \eqref{eq:prox_arate}\;
\end{itemize}
}
\BlankLine
\KwOut{A collection of samples $\left\{\Vobs^{(t)},\Vcoeff^{(t)}\right\}_{t=\Nbi+1}^{\Nmc}$ asymptotically distributed according to $f\left(\Vobs,\Vcoeff|\noisevar,\mu\right)$}
\end{algorithm}

One can notice that this algorithm boils down to successively sample according to the Gaussian likelihood distribution $f\left(\Vobs|\Vcoeff,\noisevar, \demoparam\right)$ and the conditional posterior of interest $f\left(\Vcoeff|\Vobs,\noisevar, \demoparam\right)$. This later step is achieved using either the component-wise Gibbs sampler or P-MALA technique described in paragraph \ref{subsubsec:gibbs_description_vector}. Within this framework, the generated samples $\Vcoeff^{(t)}$ should be asymptotically distributed according to the prior democratic distribution $f\left(\Vcoeff|\demoparam\right)$. Thus they can be exploited to specifically assess the validity of this step, by resorting to the properties of this distribution, see Section \ref{sec:democratic}. In this experiment, we propose to focus on one of these properties: the absolute value of the dominant component $f\left(|\coeff{\indcoef}| \big| \demoparam , \coeff{\indcoef}\in \cone_{\indcoef} \right)$ follows the gamma distribution $\calG(\dimcoeff, \lambda)$  with $\lambda=\dimcoeff\demoparam$, see \eqref{eq:dominant_pdf}. Figure~\ref{fig:geweke_exp} (left) compares the theoretical pdf of this gamma distribution with the empirical pdfs computed from $\Nmc = 2\times 10^4$ samples generated by the Geweke's scheme with the Gibbs (top) and P-MALA (bottom) samplers, where $\dimobs = \dimcoef = 3$, $\mu=2$ and $\noisevar=0.25$. Figure~\ref{fig:geweke_exp} (right) shows the corresponding quantile-quantile (Q-Q) plots which tends to ascertain the validity of the two versions of the MCMC algorithm.

\begin{figure}
\centering
\resizebox{0.75\columnwidth}{!}{%
	\includegraphics{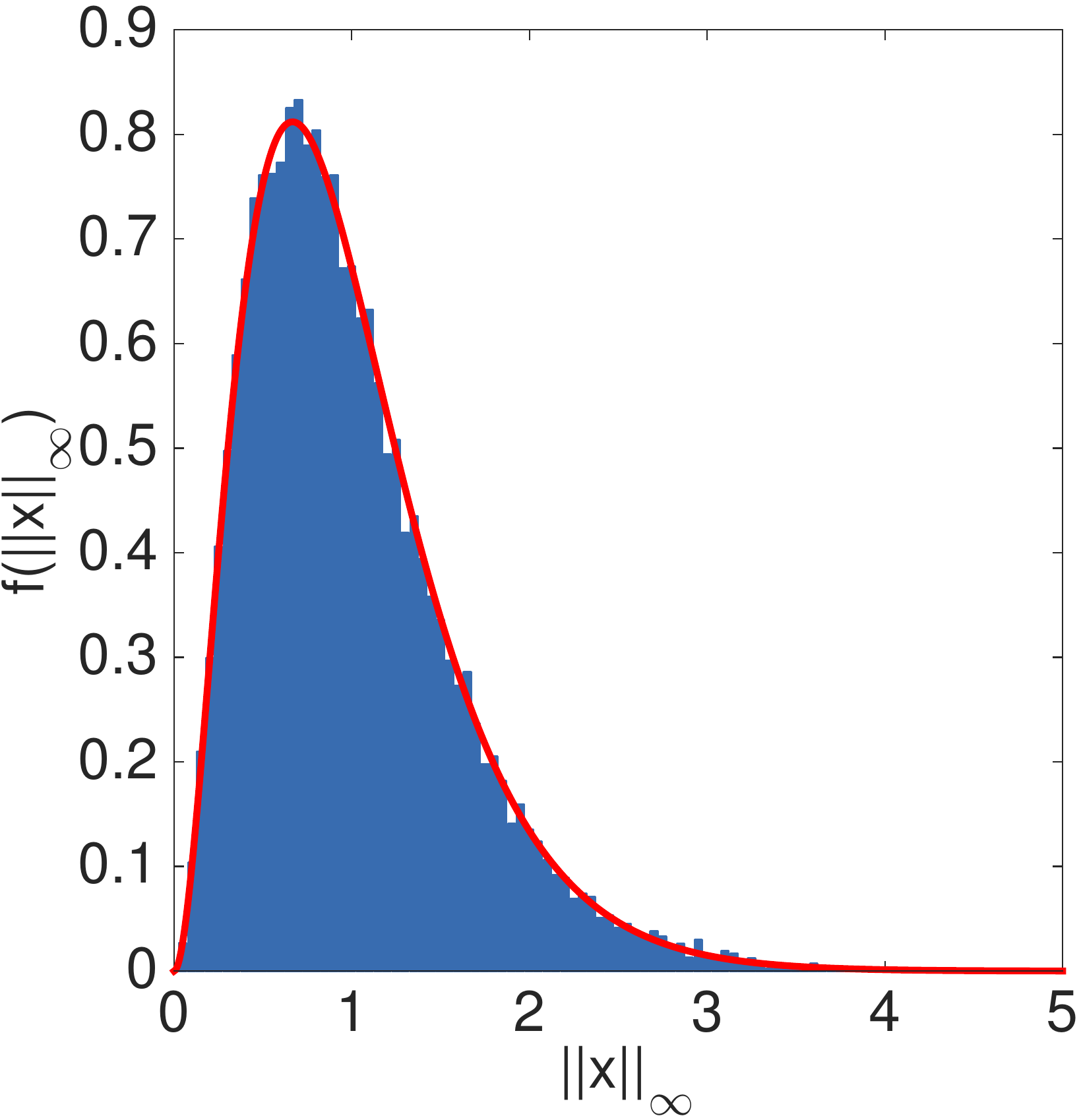} \includegraphics{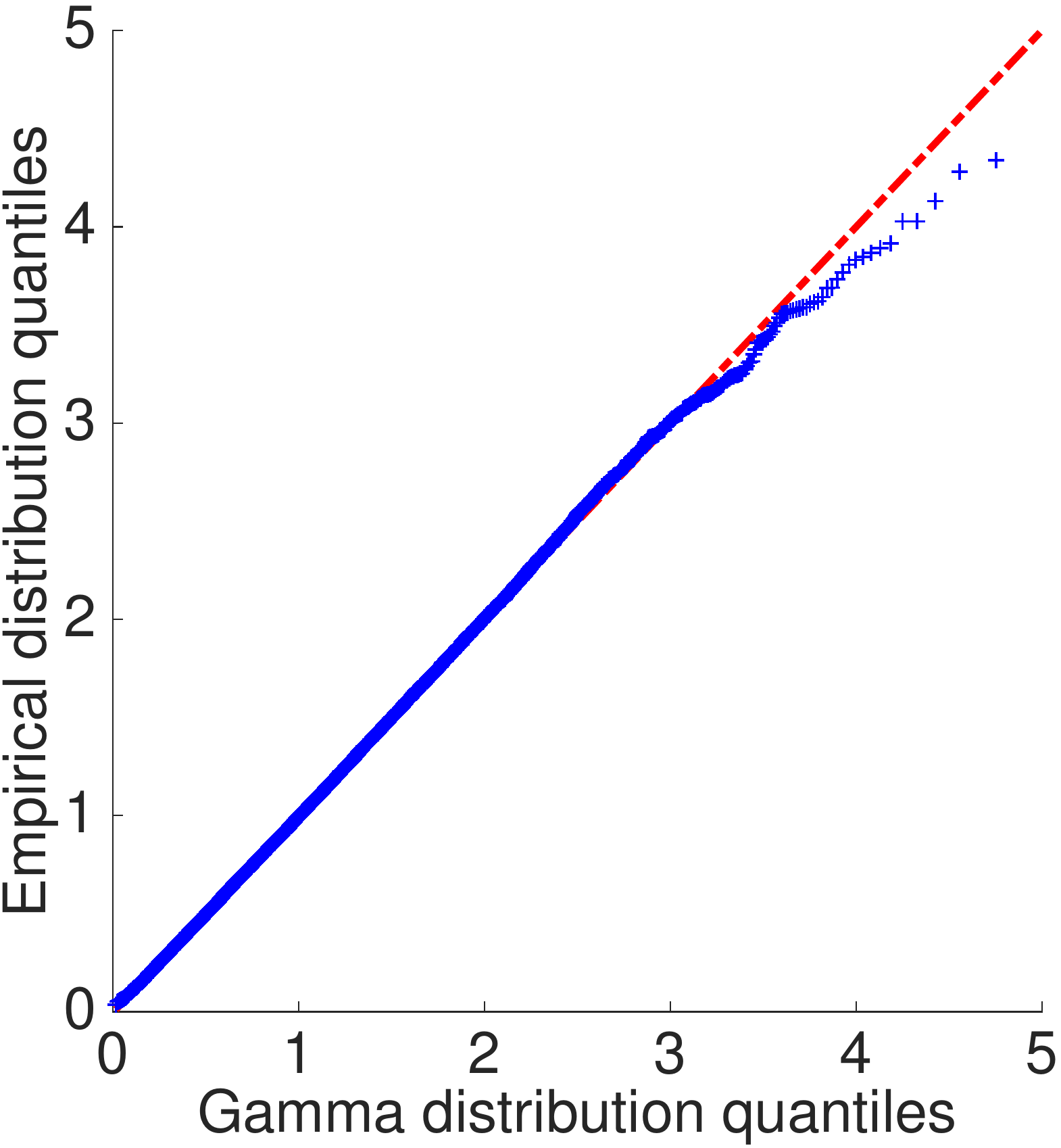}
 }
 \resizebox{0.75\columnwidth}{!}{%
	\includegraphics{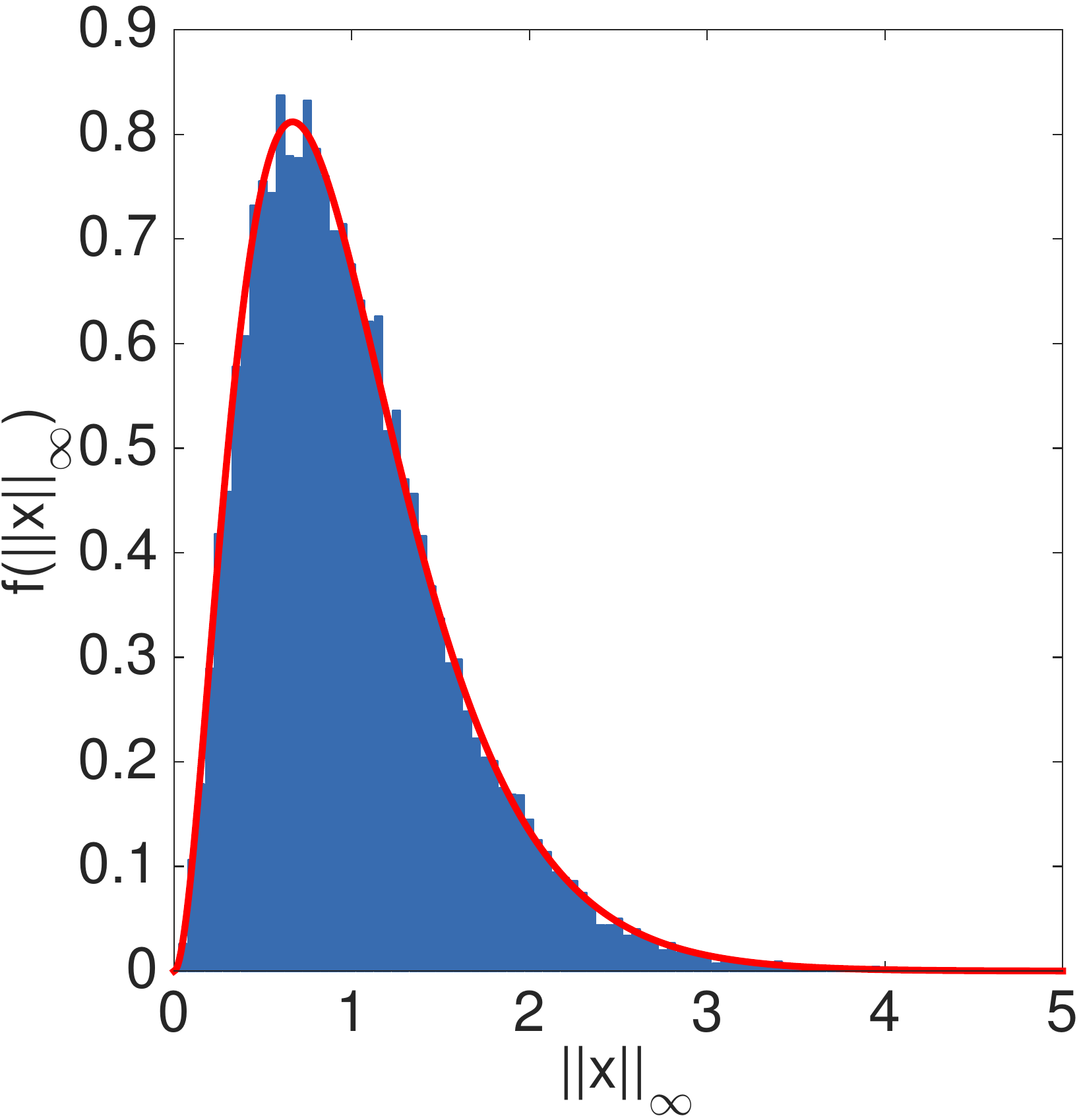} \includegraphics{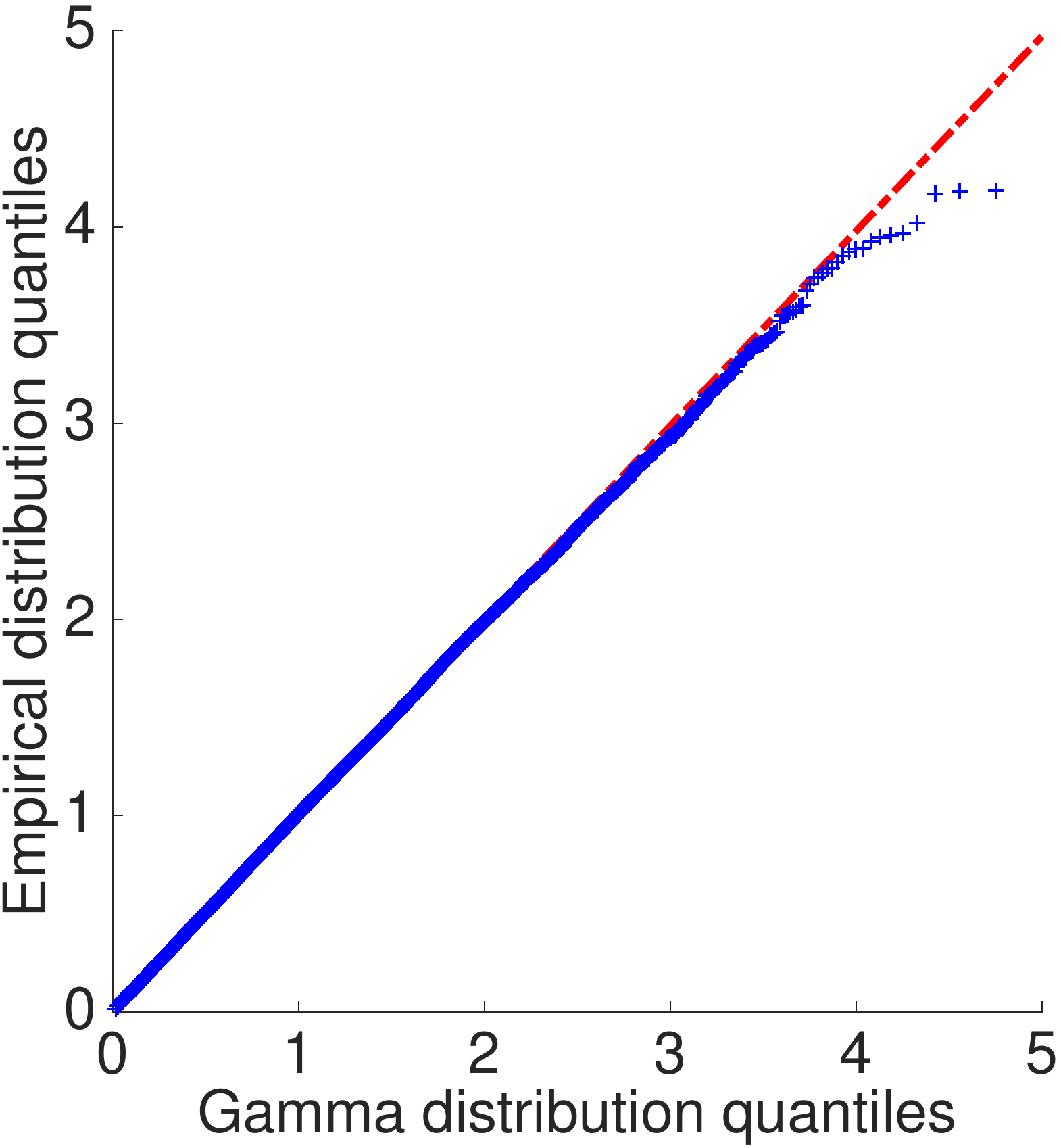}
 }
\caption{Left: theoretical Gamma pdf of the dominant component absolute value (red) and empirical pdf (blue) resulting from $2\times 10^4$ samples generated by successive conditional sampling (Algo. \ref{alg:gewekeExp}) with the Gibbs sampler (top) and P-MALA (bottom). Right: corresponding Q-Q plots.} \label{fig:geweke_exp}
\end{figure}

\subsection{Performance analysis on a toy example} \label{subsec:toy_example}

This paragraph focuses on a toy example to illustrate the convergence of the two versions of the proposed algorithm, i.e., based on Gibbs or P-MALA steps detailed in paragraphs \ref{subsubsec:gibbs_description_vector_Gibbs} and \ref{subsubsec:gibbs_description_vector_PMALA}, respectively. To this end, a simple experimental set-up with $\dimobs=\dimcoef=16$ has been considered where anti-sparse vectors $\Vcoef$ have been generated with coefficients randomly chosen among $\{-\dimcoef^{-1}, +\dimcoef^{-1}\}$. Observation vectors are then obtained following the free-noise forward model $\Vobs = \MATop \Vcoef$.

The proposed MCMC algorithms are used to generate samples $(\Vcoef^{(t)}, \sigma^{(t)}, \mu^{(t)})_{t=\Nbi}^{\Nmc}$ asymptotically distributed according to the joint posterior distribution \eqref{eq:joint_complete} with $\Nmc=3000$ iterations including $\Nbi=2000$ burn-in iterations. The MMSE and mMAP estimators of representation vector have been approximated from these samples following the strategy described in paragraph \ref{subsec:inference}. Two criteria have been used to evaluate the performance of these estimators
\begin{eqnarray}
	\snrx & = & 10 \log_{10} \frac{ \normq{\Vcoef} }{ \normq{\Vcoef - \hat{\Vcoef}} } \label{eq:def:SNRx} \\
	\text{PAPR} & = & \frac{ \dimcoef \norminf{\hat{\Vcoef}}^2 }{ \normq{\hat{\Vcoef}} } \label{eq:def:PAPR}
\end{eqnarray}
where $\hat{\Vcoef}$ refers to the MMSE or mMAP estimator of $\Vcoef$. The signal-to-noise ratio $\snrx$ measures the quality of the estimation with respect to the unknown (democratic) signal $\bfx$. Conversely, the peak-to-average power ratio PAPR quantifies anti-sparsity by measuring the ratio between the crest of the estimated signal and its average value. Note that the proposed algorithms do not aim at directly minimizing the PAPR: the use of a democratic distribution prior should promote anti-sparsity and therefore estimates with low PAPR.

\begin{figure}[h!]
\centering
 \includegraphics[width=0.75\columnwidth]{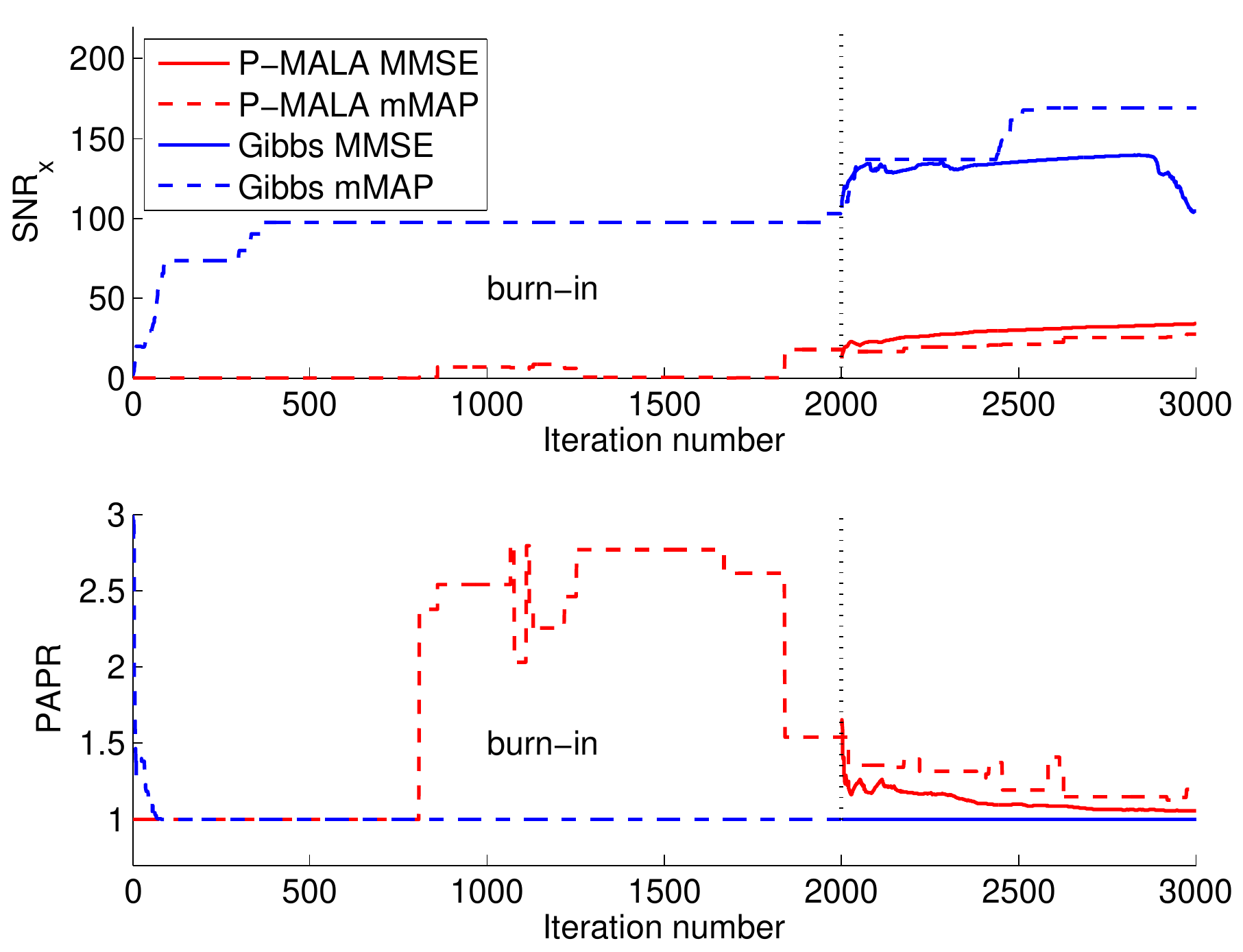}
\caption{As functions of the iteration number, $\text{SNR}_{\Vcoef}$ (top) and PAPR (bottom) associated with mMAP (dashed lines) and MMSE (continuous lines) estimates computed using the proposed algorithm based on Gibbs steps (blue) and P-MALA steps (red). The end of the burn-in period is localized with a vertical black dotted line.} \label{fig:toy_example}
\end{figure}

Figure~\ref{fig:toy_example} shows the evolution of both criteria through iterations for both versions of the algorithm. Note that, according to \eqref{eq:MMSE} and \eqref{eq:mMAP}, the mMAP estimators are approximated from all the generated samples while the MMSE estimates are only computed from samples generated after the burn-in period, located with a vertical line. The plots associated with the mMAP estimates show that, after less than $200$ iterations, the Gibbs sampler generates vectors with PAPR lower that $1.05$ and $\text{SNR}_{\Vcoef}$ higher than $75$dB. The \mmse~estimator computed from these samples quickly converges to similar results (after the burn-in period). Conversely, the estimators approximated from the samples generated using the P-MALA steps need $10$ times more iterations to converge towards solutions with PAPR around $1.2$ and $\text{SNR}_{\Vcoef}$ close to $40$dB. However, considering the computational time with a personal computer equipped with a $2.8$Ghz Intel i5 processor, the simulation of $3000$ samples requires $20$ seconds using Gibbs sampling and only  $2$ seconds using P-MALA steps. These observations highlight the fact that the algorithm based on P-MALA steps is much faster even though it needs more samples than the full Gibbs sampler to build robust estimators. To alleviate this limitation, the strategy adopted in the next experiments performs $20$ Metropolis-Hastings moves \eqref{eq:prox_arate} within a single iteration of the MCMC algorithm (as recommended in \cite{Pereyra2015prox}).

\subsection{Performance comparison} \label{subsec:perf_comparison}

\subsubsection{Experimental set-up} \label{subsubsec:perf_set_up}

In this experiment, the observation vector $\Vobs$ is composed of coefficients independently and identically distributed according to a Gaussian distribution, as in \cite{Studer2015sub}. The proposed MCMC algorithm is applied to infer the anti-sparse representation $\Vcoeff$ of this measurement vector $\Vobs$ with respect to the $\dimobs\times\dimcoeff$ coding matrix $\MATop$ for two distinct scenarios. Scenario 1 considers a small dimension problem with $\dimobs = 50$ and $\dimcoef = 70$. In Scenario 2, a higher dimension problem has been addressed, i.e., with $\dimobs = 128$ and $\dimcoef$ ranging from $128$ to $256$, which permits to evaluate the performance of the algorithm as a function of the ratio $\dimcoef/\dimobs$. In Scenario 1 (resp., Scenario 2), the proposed mMAP and MMSE estimators are computed from a total of $\Nmc=12\times10^3$ (resp., $\Nmc=55\times10^3$) iterations, including $\Nbi=10\times10^3$ (resp., $\Nbi=50\times10^3$) burn-in iterations. For this latest scenario, the algorithm based on Gibbs steps, see paragraph \ref{subsubsec:gibbs_description_vector_Gibbs}, has not been considered because of its computational burden, which experimentally justifies the interest of the proximal MCMC-based approach for large scale problems.

Algorithm performances have been evaluated over $20$ Monte Carlo simulations thanks to two figures-of-merit. The anti-sparse level of the recovered representation vector $\hat\Vcoef$ has been measured using the PAPR defined in \eqref{eq:def:PAPR}. Since the actual representation vector $\Vcoef$ is not available anymore (as in the previous experiment), the reconstruction error denoted $\snry$ and defined by
\begin{equation}
	\snry  = 10 \log_{10} \frac{ \normq{\Vobs} }{ \normq{\Vobs - \MATop \hat{\Vcoef}} }.
\end{equation}
is the second figure-of-merit.

The proposed algorithm is compared with a recent PAPR reduction technique, detailed in \cite{Studer2013jsac}. This fast iterative truncation algorithm (FITRA) is a deterministic counterpart of the proposed MCMC algorithm and solves the $\ell_{\infty}$-penalized least-squares problem \eqref{eq:problem}. Similarly to various variational techniques, FITRA needs the prior knowledge of the hyperparameters $\lambda$ (anti-sparsity level) and $\sigma^2$ (residual variance) or, equivalently, of the regularization parameter $\beta$ defined (up to a constant) as the product of the two hyperparameters, i.e., $\beta\triangleq2\lambda\sigma^2$. As a consequence, in the following experiments, this parameter $\beta$ has been chosen according to $3$ distinct rules. The first one, denoted \FITRAa, consists of applying FITRA with $\beta=2\MMSE{\lambda}\hat{\sigma}^2_{\mathrm{MMSE}}$, where $\MMSE{\lambda}$ and $\hat{\sigma}^2_{\mathrm{MMSE}}$ are the MMSE estimates obtained with the proposed P-MALA based algorithm. In the second and third configurations, the regularization parameter $\beta$ has been tuned to reach two solutions corresponding to either a target reconstruction error $\snry=20$dB (and free PAPR) or a target anti-sparsity level $\papr=1.5$ (and free $\snry$), denoted \FITRAc~and \FITRAd, respectively. For all these configurations, FITRA has been run with a maximum of $500$ iterations. Moreover, to illustrate the regularizing effect of the democratic prior (or, similarly, the $\ell_{\infty}$-penalization), the proposed algorithm and the $3$ configurations of FITRA have been finally compared with the least-squares (LS) solution as well as the MMSE and mMAP estimates resulting from a Bayesian model based on a Gaussian prior (or, similarly, an $\ell_{2}$-penalization).

\subsubsection{Results} \label{subsubsec:perf_comparison}

Table \ref{tab:lowDim} shows the results in Scenario 1 ($\dimobs = 50$ and $\dimcoef = 70$)  for all considered algorithms in terms of $\snry$ and $\papr$. For this scenario, the full Gibbs method needs approximately $6$ minutes while P-MALA needs $8$ seconds only. The \mapm~and the \mmse~estimates provided by P-MALA reach reconstruction errors of $\snry=23.7$dB and $\snry=22.4$dB, respectively. The \mapm~estimate obtained using the full Gibbs sampler performs quite similarly while numerous estimates from the Gibbs MMSE have converged to solutions that do not ensure correct reconstruction, which explains worse $\snry$ results. This is the signature of an unstable behavior of the Gibbs MMSE estimate. When using \FITRAa, solutions with similar $\snry$ but lower $\papr$ are recovered. Both MCMC and FITRA algorithms have provided anti-sparse representations with lower $\papr$ than LS or $\ell_2$-penalized solutions, which confirms the interest of the democratic prior or, equivalently, the $\ell_{\infty}$-penalization.

\begin{table}[h!]
\caption{Scenario 1: results in terms of $\snry$ and $\papr$ for various algorithms.} \label{tab:lowDim}
\centering

\begin{tabular}{l cc} 
  \toprule[1.5pt]
  	\headTab{} 	& $\snry$ & $\papr$ \\
  	\midrule
	P-MALA \mmse 			& 22.4 & 3.69 \\
	P-MALA \mapm 			& 23.7 & 2.82 \\
	Gibbs \mmse 			& 9.6 & 3.69 \\
	Gibbs \mapm 			& 12.7 & 2.82 \\
	\FITRAa 				& 21.8 & 1.53 \\
	\FITRAc 				& 19.9 & 1.86 \\
	\FITRAd 				& 9.3 & 1.5 \\
	LS 						& 306.4 & 7.27 \\
	Gaussian prior \mmse 	& 81.4 & 7.04 \\
	Gaussian prior \mapm 	& 154.6 & 6.92 \\
  \bottomrule[1.5pt]
\end{tabular}

\end{table}

Fig.~\ref{fig:compromis} displays the results  for a given realization of the measurement vector $\Vobs$ where the $\snry$ is plotted as a function of $\papr$. To provide a whole characterization of FITRA and illustrate the trade-off between the expected reconstruction error and anti-sparsity level, the solutions provided by FITRA corresponding to a wide range of regularization parameter $\beta$ are shown. The \mapm~and \mmse~solutions recovered by the two versions of the proposed algorithm are also reported in this $\snry$ vs. $\papr$ plot. They are located close to the critical region between solutions with low $\papr$ and $\snry$ and solutions with high $\papr$ and $\snry$: the proposed method recovers relevant solutions in an unsupervised way. While Gibbs estimates seem to reach a better compromise than P-MALA ones, we emphasize that the Gibbs sampler is in fact relatively unstable and therefore less robust. Moreover, the Gibbs sampler does not scale to high dimensions due to its prohibitive computational cost. \\

\begin{figure}[h!]
\centering
	\includegraphics[width=0.70\textwidth]{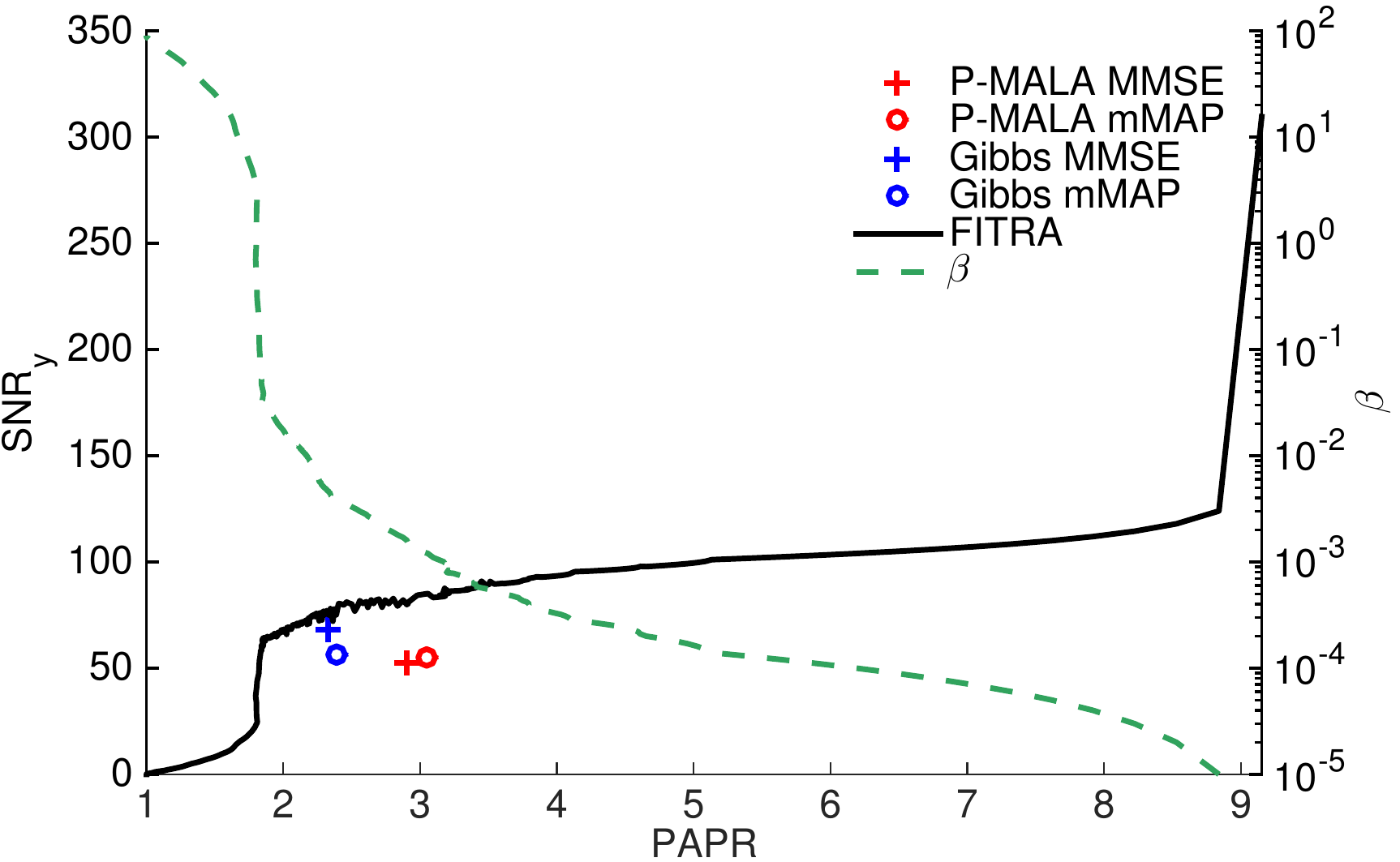}
\caption{Scenario 1: $\snry$ as a function of $\papr$. The path for the FITRA regularization parameter $\beta$ is depicted as green dashed line with the scale in the left $y$-axis.} \label{fig:compromis}
\end{figure}

Scenario 2 permits to evaluate the performances of the algorithm as a function of the ratio $\dimcoef / \dimobs$. For measurement vectors of fixed dimension $\dimobs=128$, the anti-sparse coding algorithms aim at recovering representation vectors of increasing dimensions $\dimcoef=128,\ldots,256$. As a consequence, for a given $\papr$ level of anti-sparsity, the $\snry$ is expected to be an increasing function of $\dimcoef / \dimobs$. Fig.~\ref{fig:prox_functionofRatioSNRy}  confirms this intuition since the performance of \FITRAd~in terms of $\snry$ (slightly) increases when $\dimcoef / \dimobs$ increases. Conversely, for a given level $\snry$ of reconstruction error, the $\papr$ is expected to be a decreasing function of this ratio. Fig.~\ref{fig:prox_functionofRatioPAPR} shows that the PAPR of \FITRAc decreases when $\dimcoef / \dimobs$ increases. Moreover, Fig.~\ref{fig:prox_functionofRatioSNRy} shows that, in term of $\snry$, the behavior of the P-MALA \mapm~and \mmse~estimates is similar to \FITRAa's. However, they are able to achieve lower $\papr$ once the ratio $\dimcoef / \dimobs$ is greater than $1.3$ and $1.4$, respectively, as illustrated in Fig. \ref{fig:prox_functionofRatioPAPR}.

\begin{figure}[h!]
\centering
	\includegraphics[width=0.70\textwidth]{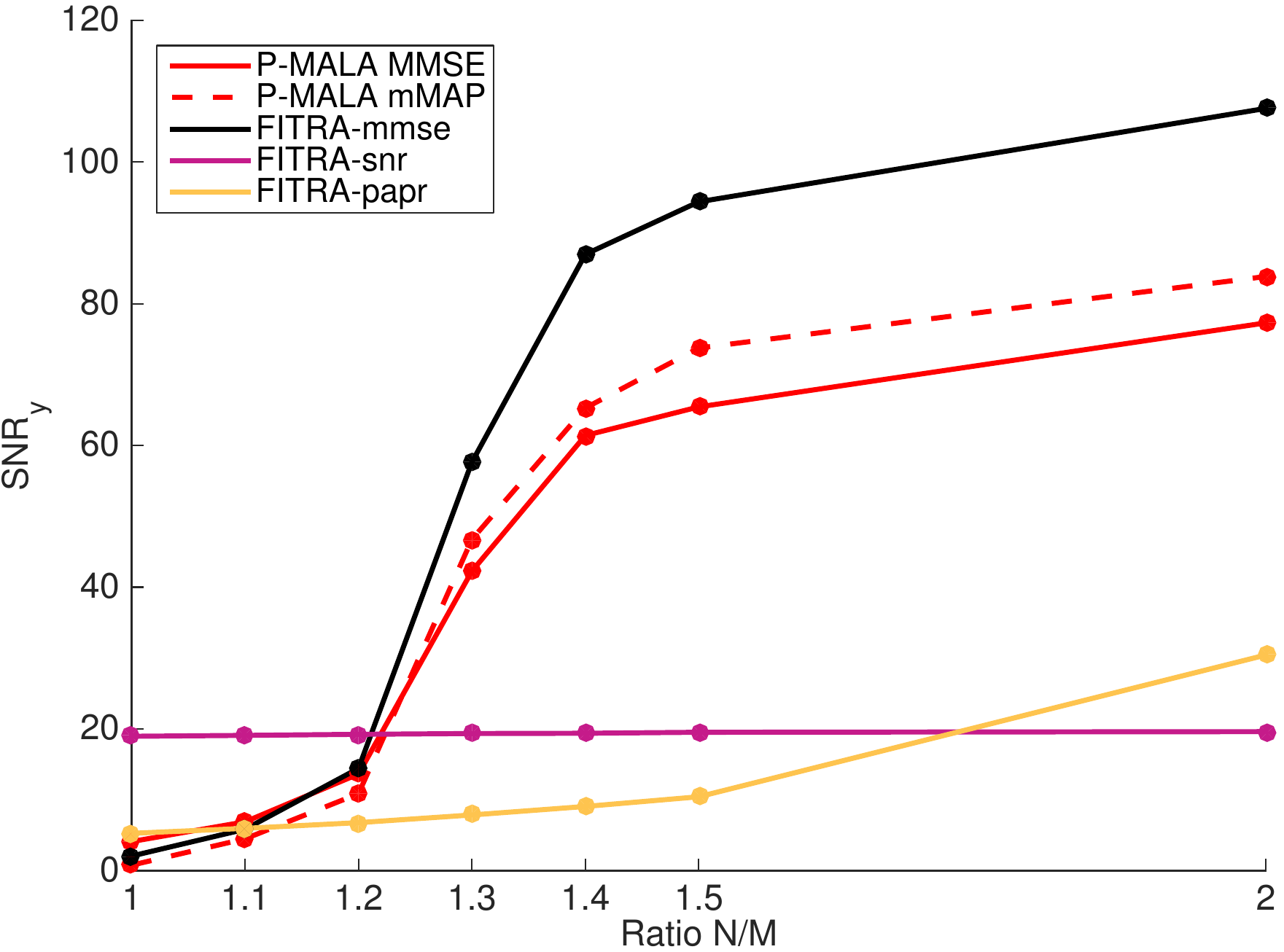}\\
\caption{Scenario 2: $\snry$ as a function of the ratio $\dimcoef / \dimobs$.} \label{fig:prox_functionofRatioSNRy}
\end{figure}

\begin{figure}[h!]
\centering
	\includegraphics[width=0.70\textwidth]{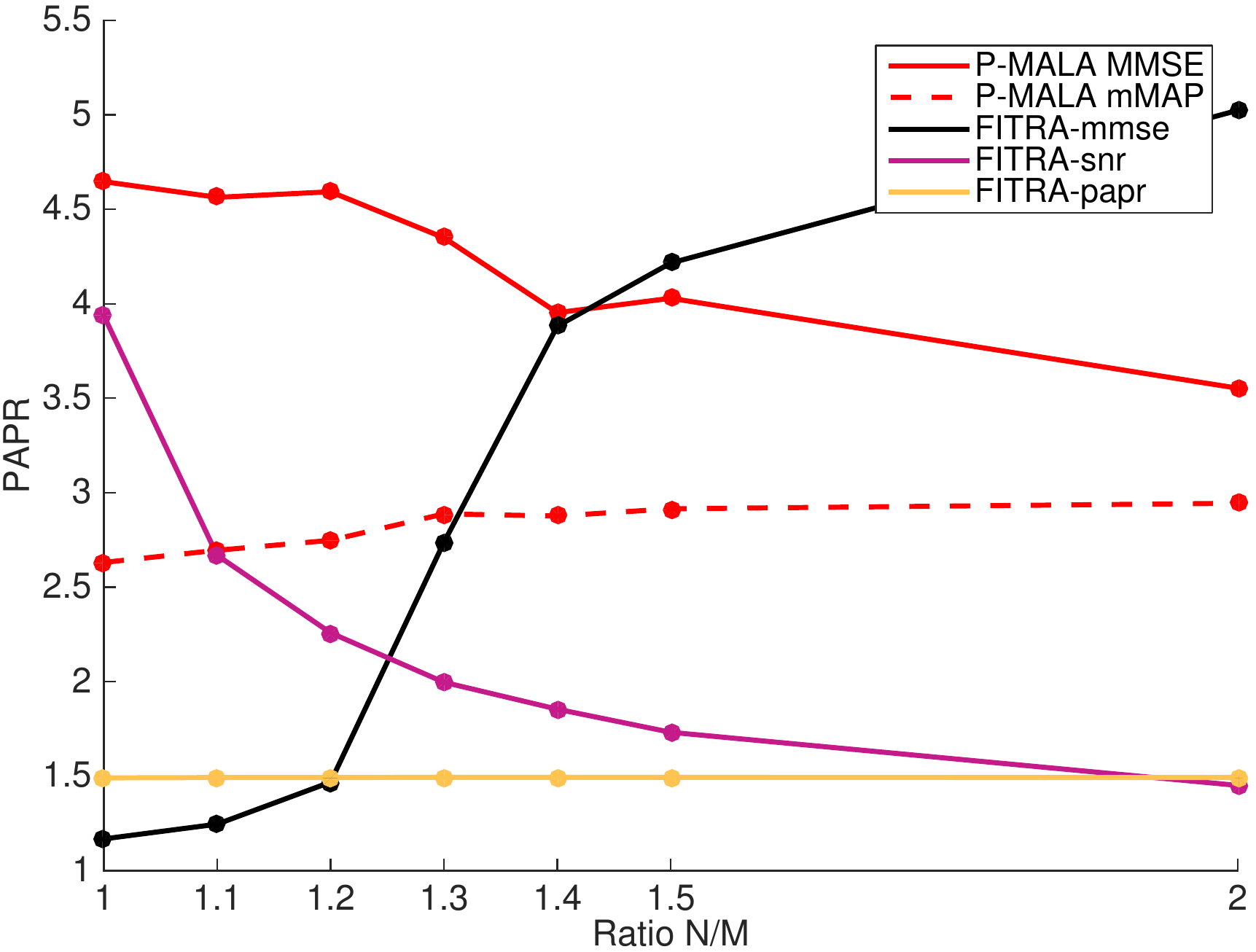}
\caption{Scenario 2: $\papr$ as a function of the ratio $\dimcoef / \dimobs$.} \label{fig:prox_functionofRatioPAPR}
\end{figure}

\section{Conclusion}
\label{sec:conclusion}

This paper introduced a fully Bayesian framework for anti-sparse coding of a given measurement vector on a known and potentially over-complete dictionary. To derive a Bayesian formulation of the problem, a new probability distribution was introduced. Various properties of this so-called \textit{democratic distribution} were exhibited, which permitted to design an exact random variate generator as well as two MCMC-based methods. This distribution was used as a prior for the representation vector in a linear Gaussian inverse problem, a probabilistic version of the anti-sparse coding problem. The residual variance as well as the anti-sparsity level were included in a fully Bayesian model, to be estimated jointly with the anti-sparse code. A Gibbs sampler was derived to generate samples distributed according to the joint posterior distribution of the coefficients of representation, the residual variance and the anti-sparse level. A second sampler was also proposed to scale to higher dimensions. To this purpose, the proximity mapping of the $\ell_{\infty}$-norm was considered to design a P-MALA within Gibbs algorithm. The generated samples were used to approximate two Bayesian estimators of the representation vector, namely the MMSE and mMAP estimators.

The validity of the proposed algorithms was first assessed following the successive conditional sampling scheme proposed in \cite{g04}. Then, they were evaluated through various experiments, and compared with FITRA a variational counterpart of the proposed MCMC algorithms. While fully unsupervised, they produced solutions comparable to FITRA in terms of reconstruction error and $\papr$, with the noticeable advantage to be fully unsupervised. In all experiments, as expected, the democratic prior distribution, was able to promote anti-sparse solutions of the coding problem. For that specific task, the mMAP estimator generally provided more relevant solutions than the \mmse~estimator. Moreover, the P-MALA-based algorithm seemed to be more robust than the full Gibbs sampler and had the ability to scale to high dimension problems, both in term of computational times and performances.

Future works include the unsupervised estimation of the coding matrix jointly with the sparse code. This would open the door to the design of encoding matrices that would ensure equal spreading of the information over their atoms. Furthermore, since the P-MALA based sampler showed promising results, it would be relevant to investigate the geometric ergodicity of the sampler. Unlike most of the illustrative examples considered in \cite{Pereyra2015prox}, this property can not be easily stated for the democratic distribution since it is not $\mathcal{C}^1$ but only $\mathcal{C}^0$.

\section*{Acknowledgments}
The authors would like to thank Dr. Vincent Mazet, Universit\'e de Strasbourg, France, for providing the code to draw according to truncated Gaussian distributions following \cite{Chopin2011}.

\appendix[Posterior distribution of the representation coefficients]
\label{app:posterior_coef}
The (hidden) mean and variances of the truncated Gaussian distributions involved in the mixture distribution \eqref{eq:hierarchical_gibbs:posterior_coef_gibbs} are given by
\begin{align}
	\mu_{1\indcoef} &= \frac{1}{\Vert \bfh_{\indcoef} \Vert^2} \left( \bfh_{\indcoef}^T \bfe_n + \noisevar \lambda \right)\nonumber \\
	\mu_{2\indcoef} &= \frac{1}{\Vert \bfh_{\indcoef} \Vert^2} \left( \bfh_{\indcoef}^T \bfe_n \right) \nonumber \\
	\mu_{3\indcoef} &= \frac{1}{\Vert \bfh_{\indcoef} \Vert^2} \left( \bfh_{\indcoef}^T \bfe_n - \noisevar \lambda \right)\nonumber
\end{align}
and
\begin{align}
	s_{\indcoef}^2 &= \frac{\noisevar}{\normq{ \bfh_{\indcoef} }}\nonumber
\end{align}
where $\bfh_i$ denotes the $i$th column of $\MATop$ and $\bfe_n =  \Vobs - \sum_{i \neq \indcoef} \coef{i} \bfh_i$. Moreover, the weights associated with each mixture component are
\begin{equation}
	\omega_{i\indcoef}  = \frac{u_{i\indcoef}}{ \sum_{j=1}^3 u_{j\indcoef} }
\end{equation}
with
\begin{align}
	u_{1\indcoef} &= \exp \left( \frac{\mu_{1\indcoef}^2}{2 s_{\indcoef}^2} + \lambda \norminf{\Vcoef_{\backslash \indcoef}} \right) \phi_{\mu_{1\indcoef}, s_{\indcoef}^2}\left(-\norminf{\Vcoef_{\backslash \indcoef}}\right)\nonumber \\
	u_{2\indcoef} &= \exp \left( \frac{\mu_{2\indcoef}^2}{2 s_{\indcoef}^2} \right) \nonumber \\
	& \quad \times \left[\phi_{\mu_{2\indcoef}, s_{\indcoef}^2}\left(\norminf{\Vcoef_{\backslash \indcoef}}\right) - \phi_{\mu_{2\indcoef}, s_{\indcoef}^2}\left(-\norminf{\Vcoef_{\backslash \indcoef}}\right)\right] \nonumber \\
	u_{3\indcoef} &= \exp \left( \frac{\mu_{3\indcoef}^2}{2 s_{\indcoef}^2} + \lambda \norminf{\Vcoef_{\backslash \indcoef}} \right) \nonumber \\
	& \quad \times \left(1 -\phi_{\mu_{3\indcoef}, s_{\indcoef}^2}\left(\norminf{\Vcoef_{\backslash \indcoef}}\right) \right) \nonumber
\end{align}
where $\phi_{\mu, s^2}(\cdot)$ is the cumulated distribution function of the normal distribution $\calN(\mu, s^2)$.

\bibliographystyle{IEEEtran}
\bibliography{strings_all_ref,all_ref,all_clement_ajout}

\end{document}